\documentclass[lettersize,journal]{IEEEtran}
\usepackage{amsmath,amssymb,amsthm}
\usepackage{graphicx}
\usepackage{epstopdf}
\usepackage{cite}
\usepackage{algorithm}
\usepackage{algpseudocode}
\usepackage{caption}
\usepackage{subcaption}

\usepackage{float}
\usepackage{tikz}
\usepackage{pgfplots}
\usepgfplotslibrary{fillbetween}
\usepgfplotslibrary{groupplots}
\input{latex_resources/mysymbol.sty}
\input{latex_resources/juan_defs.sty}

\begin{document}

\title{Learning by Transference: Training Graph Neural Networks on Growing Graphs}
\author{Juan Cervi\~no, Luana Ruiz, and Alejandro Ribeiro
	\thanks{
		J. Cervi\~no, L. Ruiz and A. Ribeiro are with the Department of Electrical and Systems Engineering, University of Pennsylvania, Philadelphia, PA 19104, USA (e-mail: \mbox{$\{$jcervino,rubruiz,aribeiro$\}$}@seas.upenn.edu).
	}
	}


\maketitle

\begin{abstract}
Graph neural networks (GNNs) use graph convolutions to exploit network invariances and learn meaningful feature representations from network data. However, on large-scale graphs convolutions incur in high computational cost, leading to scalability limitations. Leveraging the graphon---the limit object of a graph---in this paper we consider the problem of learning a graphon neural network (WNN)---the limit object of a GNN---by training GNNs on graphs sampled from the graphon. Under smoothness conditions, we show that: (i) the expected distance between the learning steps on the GNN and on the WNN decreases asymptotically with the size of the graph, and (ii) when training on a sequence of growing graphs, gradient descent follows the learning direction of the WNN. Inspired by these results, we propose a novel algorithm to learn GNNs on large-scale graphs that, starting from a moderate number of nodes, successively increases the size of the graph during training. This algorithm is further benchmarked on a decentralized control problem, where it retains comparable performance to its large-scale counterpart at a reduced computational cost.
\end{abstract}


	\section{Introduction}
	\label{sec:Intro}

	Graph neural networks (GNNs) are deep convolutional architectures formed by a succession of layers where each layer composes a graph convolution and a pointwise nonlinearity \cite{9046288,ZHOU202057}. Tailored to network data, GNNs have been used in a variety of applications such as recommendation systems \cite{fan2019graph,tan2020learning,ying2018graph,schlichtkrull2018modeling, ruiz2019invariance} and Markov chains \cite{qu2019gmnn,ruiz2019gated,li2015gated}, and fields such as biology \cite{NIPS2017_f5077839,NIPS2015_f9be311e,pmlr-v70-gilmer17a,chen2020can} and robotics \cite{qi2018learning,gama2021distributed,li2019graph}.
	Their success in these fields and applications provides ample empirical evidence of the ability of GNNs to generalize to unseen data. More recently, their successful performance has also been justified by theoretical works showing that GNNs are invariant to relabelings \cite{chen2019equivalence,keriven2019universal}, stable to graph perturbations \cite{gama2020stability}, transferable across graphs \cite{ruiz2020graphonTransferability}, and expressive \cite{kanatsoulis2022graph}.

	One of the most important features of a GNN is that, since the linear operation is a graph convolution, its number of parameters does not depend on the number of nodes of the graph. In theory, this means that GNNs can be trained on graphs of any size. But if the graph has large number of nodes, in practice training the GNN is costly because computing graph convolutions involves large matrix operations. While this issue could be mitigated by transferability~\cite{ruiz2021transferability}, i.e., by training the GNN on a smaller graph to execute it on the large graph, this approach does not give any guarantees on the distance between the optimal solutions on the small and on the large graph. In other words, when executing the GNN on the target graph we do not know if its error will be dominated by the transferability error or by the generalization error stemming from the training procedure. 
	
	
	
	
	In this paper, we propose a learning technique for training a GNN on a large graph by progressively increasing the size of the network, and we provide the theoretical guarantees under which the GNN converges to its limit architecture. We consider the limit problem of learning an ``optimal'' neural network for a graphon, which is both a graph limit and a random graph model \cite{lovasz2012large,borgs2017graphons,borgs2010moments}.  We postulate that, because sequences of graphs sampled from the graphon converge to it, the so-called graphon neural network \cite{ruiz2020graphonTransferability} can be learned by sampling graphs of growing size and training a GNN on these graphs (Algorithm \ref{alg:WNNL}). We prove that this is true in two steps. In Theorem \ref{thm:grad_Approximation}, we bound the expected distance between the gradient descent steps on the GNN and on the graphon neural network by a term that decreases asymptotically with the size of the graph. A consequence of this bias bound is that it allows us to quantify the trade-off between a more accurate gradient and one that could be obtained with less computational power. We then use this theorem to prove our main result in Theorem \ref{thm:WNN_learning}, which is stated in simplified form below.
	
	

 	\textbf{Theorem} (Graphon neural network learning, informal) Let $\bbW$ be a graphon and let $\{\bbG_n\}$ be a sequence of growing graphs sampled from $\bbW$. Consider the graphon neural network $\bbPhi(\bbW)$ and assume that it is learned by training the GNN $\bbPhi(\bbG_n)$ with loss function $\ell(\bby_n,\bbPhi(\bbG_n))$ on the sequence $\{\bbG_n\}$. Over a finite number of training steps, we obtain
 	$$
 	\|\nabla\ell(Y,\bbPhi(\bbW))\| \leq \epsilon \mbox{ with probability } 1.
 	$$
    
     The most important implication of this result is that the learning iterates computed in the sequence of growing graphs follow the direction of the graphon gradient up to a small ball, which provides theoretical support to the proposed training methodology. To further validate our algorithm, we apply it in a numerical experiment where we consider the problem of flocking. Specifically, we train GNNs to learn the actions agents need to take to flock, i.e., to fly with the same velocities while avoiding collisions\cite{tolstaya2020learning}. We compare the results obtained when progressively increasing the number of agents during training and when training is done directly on the large target graph.
     
     \subsection{Related Work}
     
     GNNs are neural information processing architectures for network data that follow from seminal works in deep learning applied to graph theory \cite{bruna2013spectral,NIPS2016_04df4d43,Gori2005ANM,Lu2003LinkbasedC}.  They have been successfully used in a wide variety of statistical learning problems \cite{kipf2016semi,SCARSELLI2018248}, where their good performance is generally attributed to the fact that they exploit invariances present in network data \cite{maron2019invariant,gama2018convolutional,chami2021machine}. More recently, a number of works show that GNNs can be transferred across graphs of different sizes \cite{ruiz2020graphonTransferability,levie2019transferability,keriven2020convergence}. Specifically, \cite{ruiz2020graphonTransferability} leverages graphons to define families of graphs within which GNNs can be transferred with an error that decreases asymptotically with the size of the graph. The papers by \cite{levie2019transferability} and \cite{keriven2020convergence} offer similar results by considering the graph limit to be a generic topological space and a random graph model respectively. In this paper, we use an extension of the transferability bound derived in \cite{ruiz2020graphonTransferability,ruiz2021transferability} to propose a novel learning algorithm for GNNs.

\section{Preliminary Definitions}	
\label{sec:Definitions}


A graph is represented as a triplet $\bbG_n=(\ccalV,\ccalE,W)$, where $\ccalV$, $|\ccalV|=n$, is the set of nodes, $\ccalE \subseteq \ccalV \times \ccalV$ is the set of edges and $W:\ccalE\to\reals$ is a map assigning weights to each edge. Alternatively, we can also represent the graph $\bbG_n$ by its graph shift operator (GSO) $\bbS_n\in \reals^{n \times n}$, a square matrix that respects the sparsity pattern of $\bbG_n$. Examples of GSOs include the adjacency matrix $\bbA$, the graph Laplacian $\bbL = \diag(\bbA \boldsymbol{1})-\bbA$ and their normalized counterparts \cite{gama2018convolutional}. In this paper we consider undirected graphs $\bbG_n$, so that $\bbS_n$ is always symmetric, and fix $\bbS_n=\bbA/n$. 

\subsection{Graph Convolutions}

Graph data is represented in the form of graph signals. A graph signal $\bbx_n=[x_1,\dots,x_n]^T\in \reals^n$ is a vector whose $i$-th component corresponds to the information present at the $i$-th node of graph $\bbG_n$. A basic data aggregation operation can be defined by applying the GSO $\bbS_n$ to graph signals $\bbx_n$ as
\begin{align}\label{eqn:graph_shift}
\bbz_n=\bbS_n\bbx_n \text{.}
\end{align}
{The value of signal $\bbx_n$ at node $i$, $x_i$, is a weighted} average of the information in the $1$-hop neighborhood of $i$, $z_i = \sum_{j \in \ccalN_i} [\bbS_n]_{ij} x_j$, where $\ccalN_i = \{j \ | \ [\bbS_n]_{ij} \neq 0\}$. Information coming from further neighborhoods (e.g., the $k$-hop neighborhood) can be aggregated by successive applications of the GSO (e.g., $\bbS_n^k\bbx_n$).

Using this notion of shift\cite{sandryhaila2013discrete}, we can define \textit{graph convolutions}. Explicitly, the graph convolutional filter with coefficients $\bbh=[h_0,\dots,h_{K-1}]$ is defined as a polynomial on the GSO $\bbS_n$,
\begin{equation} \label{eqn:graph_convolution}
\bby_n = \bbh *_{\bbS_n} \bbx_n = \sum_{k=0}^{K-1} h_k \bbS_n^k \bbx_n
\end{equation}
where $\bbx_n, \bby_n \in \reals^N$ are graph signals and $*_{\bbS_n}$ denotes the convolution operation with GSO $\bbS_n$.

Since the adjacency matrix of an undirected graph is always symmetric, the GSO admits an eigendecomposition $\bbS_n = \bbV \bbLam \bbV^H$. The columns of $\bbV$ are the GSO eigenvectors and the diagonal elements of $\bbLam$ are its eigenvalues, which take values between $-1$ and $1$ and are ordered as $-1 \leq \lambda_{-1} \leq \lambda_{-2} \leq \ldots \leq 0 \leq \ldots \leq \lambda_2 \leq \lambda_1 \leq 1$. 
Since the eigenvectors of $\bbS_n$ form an orthonormal basis of $\reals^n$, we can project \eqref{eqn:graph_convolution} onto this basis to obtain the spectral representation of the graph convolution,
\bal \label{eqn:spectral_graph}
h(\lambda) = \sum_{k=0}^{K-1} h_k \lambda^k \text{.}
\eal
Importantly, note that \eqref{eqn:spectral_graph} only depends on the $h_k$ and on the eigenvalues of the GSO. {Hence, by} the Cayley-Hamilton theorem, convolutional filters may be used to represent any graph filter with spectral representation $h(\lambda) = f(\lambda)$ where $f$ is analytic \cite{strang1976}.

\subsection{Graph Neural Networks}

Graph neural networks are layered architectures where each layer consists of a graph convolution followed by a pointwise nonlinearity $\rho$.
Layer $l$ can output multiple features $\bbx^f_{nl}$, $1 \leq f \leq F_{l}$, which we stack in a matrix $\bbX_l=[\bbx_{nl}^1,\dots,\bbx^{F_l}_{nl}]\in \reals^{n\times F_l}$. Each column of the feature matrix is the value of the graph signal corresponding to feature $f$. To map the $F_{l-1}$ features coming from layer $l-1$ into $F_l$ features, $F_{l-1} \times F_l$ convolutions need to be computed, one per input-output feature pair. Stacking their weights in $K$ matrices $\bbH_{lk}\in \reals^{F_{l-1}\times F_l}$, we may write the $l$-th layer of the GNN as
\begin{equation} \label{eqn:gcn_layer}
\bbX_{l} = \rho \left( \sum_{k=0}^{K-1} \bbS_n^k \bbX_{l-1} \bbH_{lk} \right). 
\end{equation}
The GNN input (i.e., the input of layer $l=1$) is given by $\bbX_0 = \bbX \in\reals^{n \times F_0}$ and, in an $L$-layer GNN, the operation in \eqref{eqn:gcn_layer} is cascaded $L$ times to obtain the GNN output $\bbY = \bbX_L$. In this paper we assume $F_0=F_L =1$ so that $\bbY= \bby_n$ and $\bbX = \bbx_n$. A more concise representation of this GNN can be obtained by grouping all learnable parameters $\bbH_{lk}$ in a tensor $\ccalH=\{\bbH_{lk}\}_{l,k}$ and defining the map $\bby_n=\bbPhi(\bbx_n;\ccalH,\bbS_n)$. 
Due to the polynomial nature of the graph convolution, the dimensions of the learnable parameter tensor $\ccalH$ are independent from the size of the graph ($K$ is typically much smaller than $n$). Ergo, a GNN trained on a graph $\bbG_n$ can be deployed on a network $\bbG_m$ with $m \neq n$. 

\subsection{Graphons}


\begin{figure*}[t]
\centering
\begin{subfigure}{0.22\textwidth}
\centering
\includegraphics[width=3.7cm]{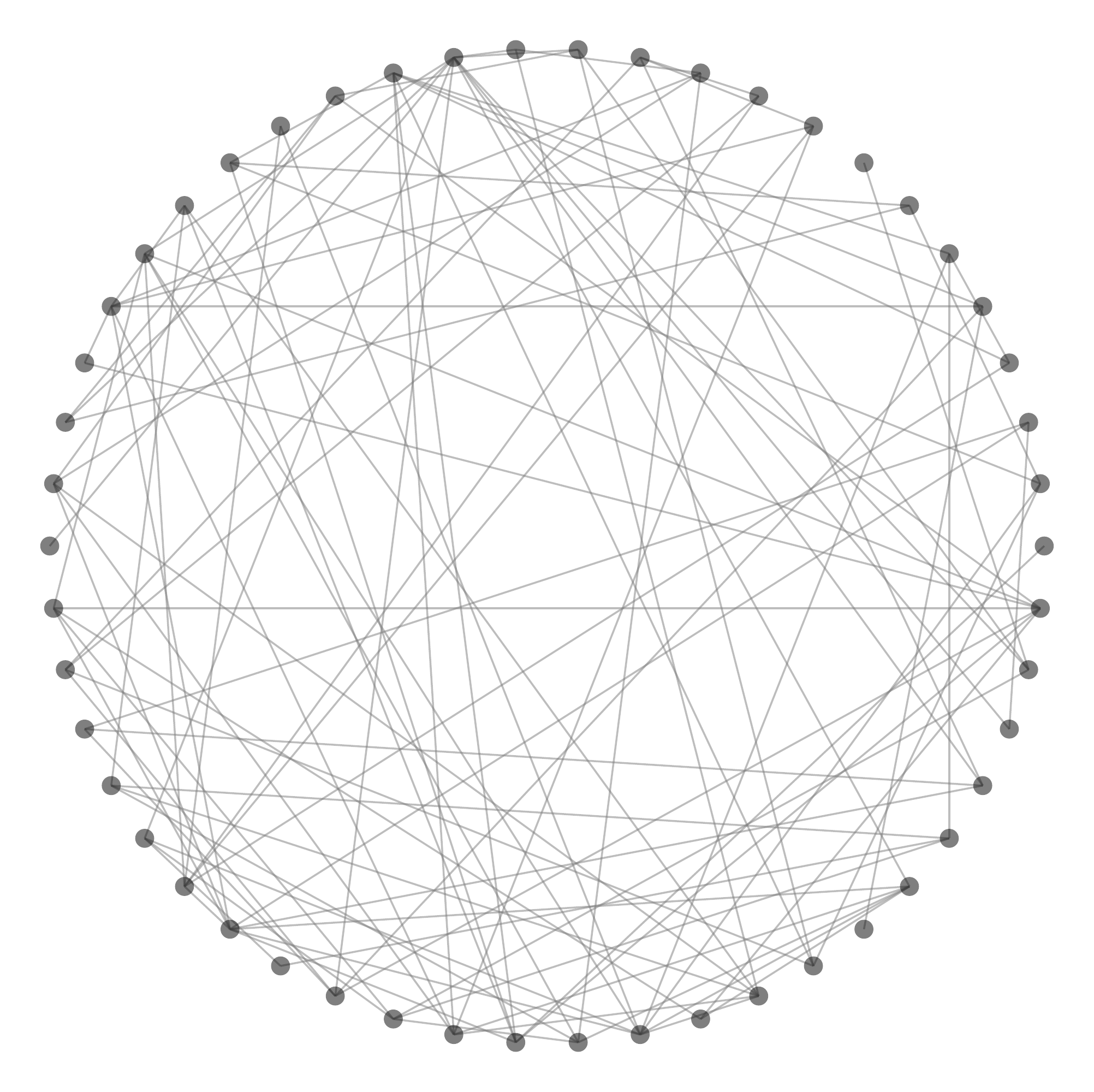}
\caption{50 nodes}
\label{fig:50nodes}
\end{subfigure}
\hspace{0.3cm}
\begin{subfigure}{0.22\textwidth}
\centering
\includegraphics[width=3.7cm]{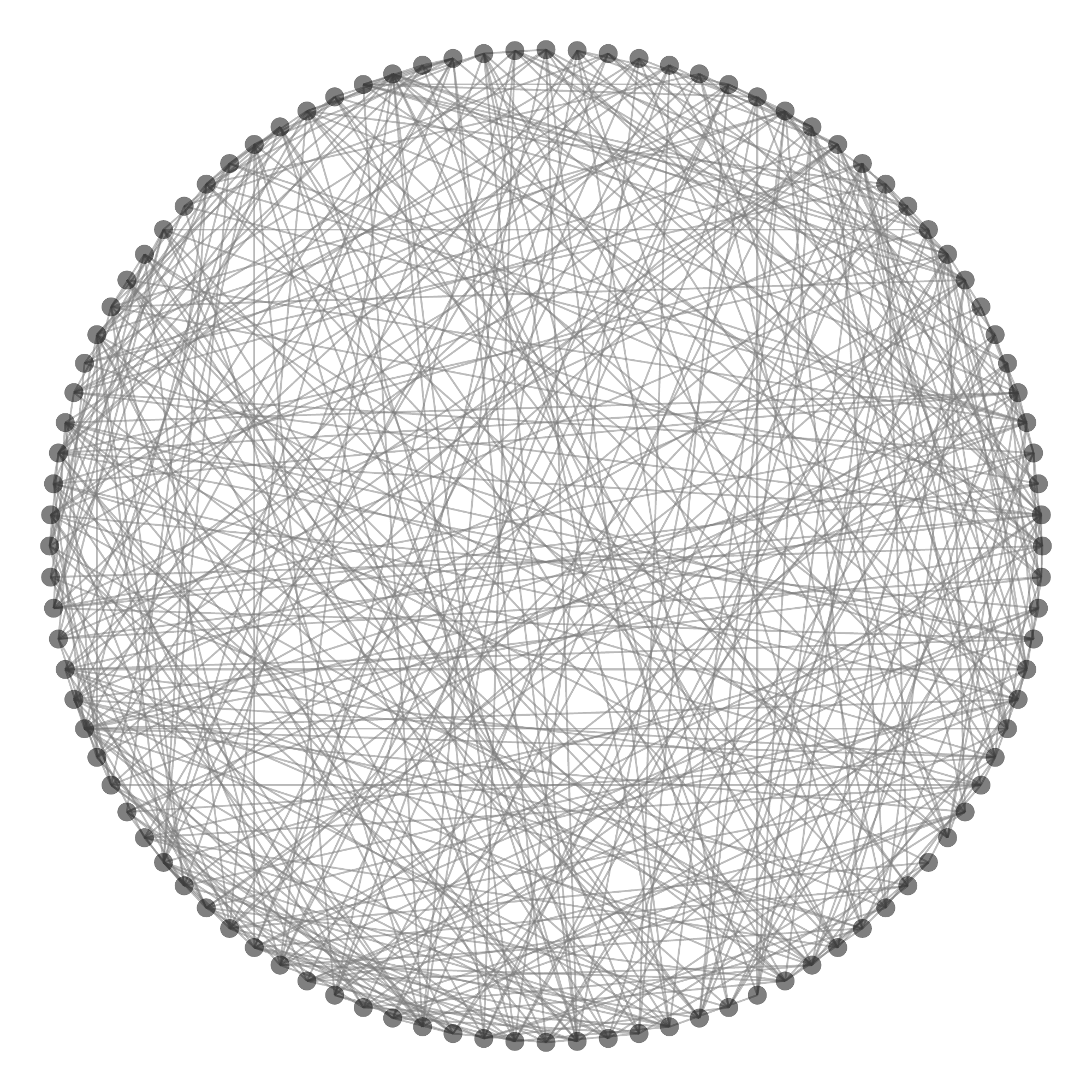}
\caption{100 nodes}
\label{fig:100nodes}
\end{subfigure}
\hspace{0.3cm}
\begin{subfigure}{0.22\textwidth}
\centering
\includegraphics[width=3.7cm]{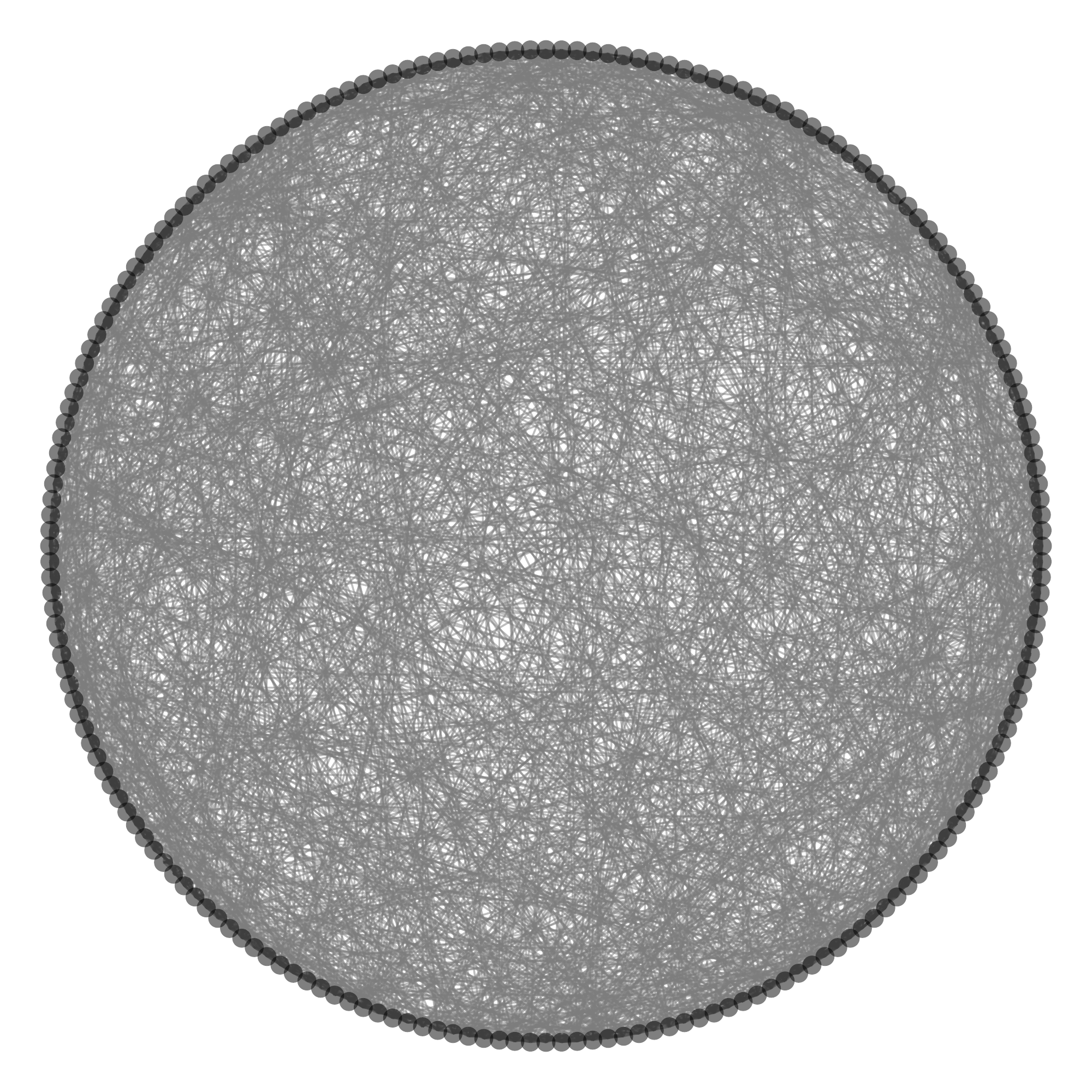}
\caption{200 nodes}
\label{fig:200nodes}
\end{subfigure}
\hspace{0.3cm}
\begin{subfigure}{0.22\textwidth}
\centering
\includegraphics[width=3.7cm]{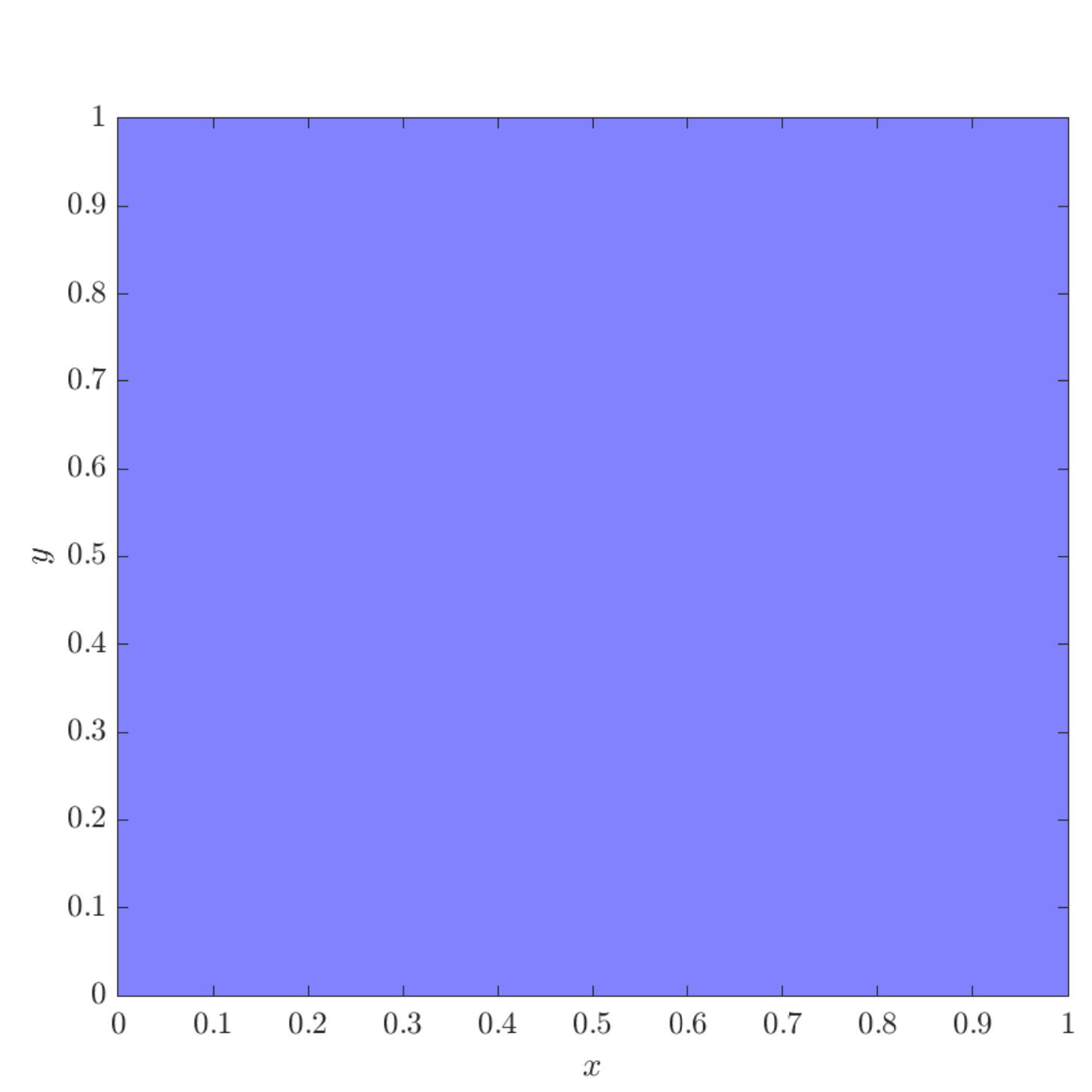}
\caption{Graphon}
\label{fig:er_graphon}
\end{subfigure}

\caption{Sequence of stochastic graphs with 50 (a), 100 (b), and 200 nodes (c) converging to the constant graphon (d).}
\label{fig:graphon_example}
\end{figure*}

A graphon is a bounded, symmetric, and measurable function $\bbW:[0,1]^2\to [0,1]$ which has two theoretical interpretations; it is both a graph limit and a generative model for graphs. In the first interpretation, sequences of dense graphs converge to a graphon in the sense that the densities of adjacency-preserving graph motifs converge to the densities of these same motifs on the graphon \cite{lovasz2012large}. A graphical example of the convergence of dense graphs to a graphon is depicted in figure \ref{fig:graphon_example}. In the second, graphs can be generated from a graphon by sampling points $u_i, u_j$ from the unit interval and either assigning weight $\bbW(u_i,u_j)$ to edges $(i,j)$, or sampling edges $(i,j)$ with probability $\bbW(u_i,u_j)$. 
In this paper, we focus on graphs $\bbG_n$ where the points $u_i$ are defined as $u_i = (i-1)/n$ for $1 \leq i \leq n$ and where the adjacency matrix $\bbS_n$ is sampled from $\bbW$ as
\begin{align} \label{eqn:BernoulliSampledGraph}
[\bbS_n]_{ij} &\sim \text{Bernoulli}([\mbS_n]_{ij}),\text{with }{[\mbS_n]_{ij}} = \bbW(u_i,u_j)
\end{align}
We refer to the graph $\bbS_n$ as a \textit{stochastic graph}, and to the graph $\mbS_n$ from which its edges are sampled as the $n$-node \textit{template graph} associated with the graphon $\bbW$. Sequences of stochastic graphs generated in this way can be shown to converge to $\bbW$ with probability one \cite{lovasz2012large}[Chapter 11].




In practice, the two theoretical interpretations of a graphon allow thinking of it as an identifying object for a \textit{family} of graphs of different sizes that are structurally similar. Hence, given a network we can use its family's identifying graphon as a continuous proxy for the graph. This is beneficial because it is typically easier to operate in continuous domains than in discrete domains, even more so if the network is large. In the following, we will leverage these ideas to consider graphon data and graphon neural networks as proxies for graph data and GNNs supported on graphs of arbitrary size.



\subsection{Graphon Convolutions}

{A graphon signal is defined as a function $X \in L^2([0,1])$. Analogously to graph signals, graphon signals can be diffused by application of a linear} integral operator given by
\begin{equation}
T_{\bbW}X(v) = \int_0^1 \bbW(u,v)X(u)du.
\end{equation}
which we call the graphon shift operator (WSO). Since $\bbW$ is bounded and symmetric, $T_{\bbW}$ Hilbert-Schmidt and self-adjoint \cite{lax02-functional}. As such, $\bbW$ can be written as $\bbW(u,v) = \sum_{i \in \mbZ\setminus\{0\}} \lambda_i \varphi_i(u)\varphi_i(v)$ where $\lambda_i$ are the graphon eigenvalues and $\varphi_i$ the graphon eigenfunctions. The eigenvalues have magnitude at most one and are ordered as $-1 \leq \lambda_{-1} \leq \lambda_{-2} \leq \ldots\leq 0 \leq \ldots \lambda_2 \leq \lambda_1 \leq 1$. The eigenfunctions form an orthonormal basis of $L^2([0,1])$. 

The notion of a graphon shift allows defining graphon convolutions and specifically, we define the graphon convolutional filter with coefficients $[h_0,\dots,h_{K-1}]$ as
\begin{align}\begin{split} \label{eq:graphon_convolution}
&Y =T_\bbh X = \bbh *_{\bbW} X = \sum_{k=0}^{K-1} h_k (T_{\bbW}^{(k)} X)(v) \quad \mbox{with} \\
&(T_{\bbW}^{(k)}X)(v) = \int_0^1 \bbW(u,v)(T_\bbW^{(k-1)} X)(u)du
\end{split}\end{align}
where $X, Y \in \L^2([0,1])$ are graphon signals, $*_{\bbW}$ denotes the convolution operation with graphon $\bbW$ and $T_{\bbW}^{(0)}=\bbI$ is the identity \cite{ruiz2020graphonFourierTransform}.
Projecting the filter \eqref{eqn:graph_convolution} onto the eigenbasis $\{\varphi_i\}_{i \in \mbZ\setminus\{0\}}$, we see that the graphon convolution admits a spectral representation given by
\bal
h(\lambda) = \sum_{k=0}^{K-1} h_k \lambda^k \text{.}
\eal
Like its graph counterpart, the spectral representation of the graphon convolution only depends on the coefficients $\bbh$ and on the eigenvalues of the graphon.  

\subsection{Graphon Neural Networks}

Graphon neural networks (WNNs) are the extensions of GNNs to graphon data. In the WNN, each layer consists of a bank of graphon convolutions \eqref{eq:graphon_convolution} followed by a nonlinearity $\rho$. Assuming that layer $l$ maps $F_{l-1}$ features into $F_l$ features, the parameters of the $F_{l-1} \times F_l$ convolutions \eqref{eq:graphon_convolution} can be stacked into $K$ matrices $\{\bbH_{lk}\}\in \reals^{F_{l-1}\times F_l}$. Then, we can write the $f$th feature of the $l$th layer as 
\begin{equation}\label{eqn:wcn_layer}
X_{l}^{f} = \rho\left(\sum_{g=1}^{F_{l-1}} \sum_{k=1}^{K-1} (T_{\bbW}^{(k)} X^g_{l-1})[\bbH_{lk}]_{gf} \right)
\end{equation}
for $1 \leq f \leq F_{l}$. 

For an $L$-layer WNN, $X_0^g$ is given by the input data $X^g$ for $1 \leq g \leq F_0$, \eqref{eqn:wcn_layer} is repeated for $1 \leq \ell \leq L$, and the WNN output is given by $Y^f = X_L^f$. In this paper, we fix $F_L=F_0=1$ so that $X_L = Y$ and $X_0 = X$. Like the GNN, the WNN can be represented more succinctly as a map $Y = \bbPhi(X;\ccalH,\bbW)$. The tensor $\ccalH = \{\bbH_{lk}\}_{l,k}$ groups the filter coefficients at all layers and for all features of the WNN. 



From the representation of the GNN and the WNN as maps $\bbPhi(\bbx_n;\ccalH,\bbS_n)$ and $\bbPhi(X;\ccalH,\bbW)$, we see that these architectures can share the same filter coefficients $\ccalH$. this observation is important because, since graphs can be sampled from graphons as in \eqref{eqn:BernoulliSampledGraph}, it implies that we can similarly use the WNN $\bbPhi(X;\ccalH,\bbW)$ to \textit{sample} GNNs
\begin{align} \label{eqn:gcn_obtained}
\bby_n = \bbPhi(\bbx_n;\ccalH,\bbS_n) \mbox{ where }&[\bbS_n]_{ij} \sim \mbox{Bernoulli}(\bbW(u_i,u_j))\nonumber \\
&[\bbx_n]_i = X(u_i) \text{.}
\end{align}
In other words, the WNN can be interpreted as a generative model for GNNs $ \bbPhi(\bbx_n;\ccalH,\bbS_n)$. 

Conversely, a WNN can be induced by a GNN. Given a GNN $\bby_n = \bbPhi(\bbx_n;\ccalH,\bbS_n)$, the WNN induced by this GNN is defined as 
\begin{align} \label{eqn:wnn_induced}
\begin{split}
Y_n &= \bbPhi(X_n;\ccalH,\bbW_n) \mbox{ where }\\
\bbW_n(u,v) &= \sum_{i=1}^n\sum_{j=1}^n [\bbS_n]_{ij} \mbI(u \in I_i) \mbI(v \in I_j) \\
X_n(u) &= \sum_{i=1}^n [\bbx_n]_i \mbI(u \in I_i)
\end{split}
\end{align}
where $\mbI$ denotes the indicator function and the intervals $I_i$ are defined as $I_i=[(i-1)/n,i/n)$ for $1 \leq i \leq n-1$ and $I_n = [(n-1)/n,1]$. The graphon $\bbW_n$ and the signals $X_n$ and $Y_n$ are respectively the \textit{graphon induced by the graph} $\bbG_n$ and the \textit{graphon signals induced by the graph signals} $\bbx_n$ and $\bby_n$.

	\section{Graphon Empirical Learning}
	\label{sec:WLearning}
	
	
	{In this paper, our goal is to solve statistical learning problems where the data is supported on very large graphs. We restrict the optimization domain to the function class of graphon neural networks (cf. \eqref{eqn:wcn_layer}) both to avoid trivial solutions and to be able to generalize to unseen data while leveraging invariances of the graph. Explicitly, we aim to solve an optimization problem of the form}
	
	\begin{align}\label{eqn:SRM}
	\underset{\ccalH}{\text{minimize}} \quad \mbE_{p(Y,X)}[\ell(Y,\bbPhi(X;\ccalH,\bbW))]
	\end{align}
	where $p(Y,X)$ is an unknown joint data distribution, $\ell$ a {non-negative} loss function and $\bbPhi(X;\ccalH,\bbW)$ a WNN parametrized by the graphon $\bbW$ and by a set of learnable weights $\ccalH$. In this paper, we consider non-negative loss functions $\ell: \reals\times\reals\to\reals^{+}$, such that $\ell(x,y)=0$ if and only if $x=y$. Because the joint probability distribution $p(Y,X)$ is unknown, we cannot derive a closed-form solution of \eqref{eqn:SRM}, but this problem can be approximated by the empirical risk minimization (ERM) problem over graphons.
	
	\subsection{Graphon Empirical Risk Minimization}
	
	 Suppose that we have access to samples of the distribution $\ccalD=\{(X^j,Y^j)\sim p(X,Y), j=1,\dots,|\ccalD|\} $. Provided that these samples are obtained independently and that $|\ccalD|$ is sufficiently large, the statistical loss in \eqref{eqn:SRM} can be approximated by {the empirical sum over $\ccalD$,}
	 	\begin{align}\label{eqn:ERM}
	\underset{\ccalH}{\text{minimize}} \quad \sum_{j=1}^{|\ccalD|}\ell(Y^j,\bbPhi(X^j;\ccalH,\bbW))
	\end{align}
	 giving way to the ERM problem \cite{hastie2009elements,shalev2014understanding,vapnik1999nature,kearns1994introduction}. To solve this problem using local information, we {can use} gradient descent \cite{goodfellow2016deep}. The $k$th gradient descent iteration is given by
	\bal
	\ccalH_{k+1}&=\ccalH_{k}-\eta_k \frac{1}{|\ccalD|}\sum_{j=1}^{|\ccalD|}\nabla_{\ccalH} \ell(Y^j,\bbPhi(X^j;\ccalH_k,\bbW))\label{eqn:WNN_Learning_Step}
	\eal
	where $\eta_k\in (0,1)$ is the step size at iteration $k$. 
	
	In practice, the gradients on the right hand side of \eqref{eqn:WNN_Learning_Step} cannot be computed because the graphon $\bbW$ is an unknown limit object.
	However, we can exploit the fact that $\bbW$ is a random graph model to approximate the gradients $\nabla_{\ccalH} \ell(Y^j,\bbPhi(X^j;\ccalH,\bbW))$ by sampling stochastic graphs $\bbG_n$ with GSO $\bbS_n$ [cf. \eqref{eqn:BernoulliSampledGraph}] and calculating $\nabla_{\ccalH} \ell(\bby^j_n,\bbPhi(\bbx^j_n;\ccalH,\bbS_n))$, where $\bbx_n^j$ and $\bby_n^j$ are respectively graph signals and graph labels as in \eqref{eqn:gcn_obtained}. In this case, the graphon empirical learning step in \eqref{eqn:WNN_Learning_Step} becomes 
	\bal
	\ccalH_{k+1}&=\ccalH_{k}-\eta_k  \frac{1}{|\ccalD|}\sum_{j=1}^{|\ccalD|} \nabla_{\ccalH} \ell(\bby^j_{n},\bbPhi(\bbx^j_{n};\ccalH_k,\bbS_n)). \label{eqn:GNN_Learning_Step}
	\eal

	As we show in the next subsection, the gradient descent iterates on the graph \eqref{eqn:GNN_Learning_Step} are good approximations of the gradient descent steps on the graphon \eqref{eqn:WNN_Learning_Step}.

\subsection{Gradient Approximation}


In order to derive an upper bound for the expected error made when using the graph gradients \eqref{eqn:GNN_Learning_Step} to approximate the graphon gradients \eqref{eqn:WNN_Learning_Step}, we introduce a Lipschitz smoothness definition, as well as two definitions related to the spectra of the graphon $\bbW$ and the graphon $\bbW_n$ induced by $\bbS_n$ (cf. equation \ref{eqn:wnn_induced}).

\begin{definition}[Lipschitz continuous functions] \label{def:norm_lips}
	A function $f(u_1, u_2, \ldots, u_d)$ is $A$-Lipschitz on the variables $u_1, \ldots, u_d$ if it satisfies $|f(v_1,v_2,\ldots,v_d)-f(u_1,u_2,\ldots,u_d)| \leq A\sum_{i=1}^d |v_i-u_i|$.
	If $A=1$, we say that this function is normalized Lipschitz.
\end{definition}

\begin{definition}[$c$-band cardinality of $\bbW$]
\label{def:c_band_cardinality}
{The $c$-band cardinality, denoted $B_{\bbW}^c$, is the number of eigenvalues whose absolute value is larger than $c$, }
\bal
B_{\bbW}^c = \# \{\lambda_i:\|\lambda_i\|\leq c \}\nonumber
\eal
\end{definition}
\begin{definition}[$c$-eigenvalue margin of $\bbW$ - $\bbW_n$]
\label{def:c_eigenvalue_margin}
{The $c$-eigenvalue margin of $\bbW$ and $\bbW_n$ is the minimum distance between eigenvalues $\lambda_i(T_\bbW)$ and $\lambda_j(T_{\bbW_n})$, $i \neq j$, }
\bal
\delta_{\bbW\bbW_n}^c = \min_{i,j\neq i}  \{ \| \lambda_i(T_\bbW) - \lambda_i(T_{\bbW_n})\| : \| \lambda_i(T_{\bbW_n})\|\geq c \}\nonumber
\eal
\end{definition}

We also impose the following Lipschitz continuity assumptions (AS\ref{as1}--AS\ref{as4}), and an assumption on the size of the graph $\bbS_n$ (AS\ref{as5}).

\begin{assumption} \label{as1}
	The graphon $\bbW$ and the graphon signals $X$ and $Y$ are normalized Lipschitz.
\end{assumption} 
\begin{assumption} \label{as2}
	The convolutional filters $h$ are normalized Lipschitz and non-amplifying, i.e., $\|h(\lambda)\|<1$.
\end{assumption} 
\begin{assumption} \label{as3}
	The activation functions and their gradients are normalized Lipschitz, and $\rho(0)=0$.
\end{assumption} 
\begin{assumption} \label{as4}
	The loss function $\ell: \reals\times\reals\to\reals^{+}$ and its gradient are normalized Lipschitz, and $\ell(x,x)=0$.
\end{assumption}
\begin{assumption} \label{as5}
	For a fixed value of $\xi \in (0,1)$, $n$ is such that $n-{\log (2n /\xi)}/{d_\bbW}> {2}/{d_\bbW}$ where $d_\bbW$ denotes the maximum degree of the graphon $\bbW$, i.e., $d_\bbW=\max_v \int^1_0 \bbW(u,v)du$.
\end{assumption} 

Under AS\ref{as1}--AS\ref{as5}, the expected norm of the difference between the graphon and graph gradients in \eqref{eqn:WNN_Learning_Step} and \eqref{eqn:GNN_Learning_Step} is upper bounded by Theorem \ref{thm:grad_Approximation}. The proof is deferred to the Appendix \ref{sec:AppendixA}.



\begin{theorem}\label{thm:grad_Approximation}
	Consider the ERM problem in \eqref{eqn:ERM} and let $\bbPhi(X;\ccalH,\bbW)$ be an $L$-layer WNN with $F_0=F_L=1$, and $F_{l}=F$ for $1\leq l \leq L-1$.  Let $c \in (0,1]$ and assume that the graphon convolutions in all layers of this WNN have $K$ filter taps [cf. \eqref{eq:graphon_convolution}]. Let $\bbPhi(\bbx_n;\ccalH,\bbS_n)$ be a GNN sampled from $\bbPhi(X;\ccalH,\bbW)$ as in \eqref{eqn:gcn_obtained}. Under assumptions AS\ref{as1}--AS\ref{as5}, it holds
	\bal\label{eq:thm_1}
	\mbE[\|\nabla_{\ccalH}&\ell(Y,\bbPhi(X;\ccalH,\bbW))-\nabla_{\ccalH}\ell(Y_n,\bbPhi(X_n;\ccalH,\bbW_n))\|]\nonumber\\
	&\leq \gamma c+\ccalO\bigg( \sqrt{\frac{\log(n^{3/2})}{n}}\bigg)
	\eal
	where $Y_n$ is the graphon signal induced by $[\bby_n]_i = Y(u_i)$ [cf. \eqref{eqn:wnn_induced}], and $\gamma$ is a constant that depends on the number of layers $L$, features $F$, and filter taps $K$ of the GNN [cf. Definition \ref{def:Constant_Of_GNN} in the Appendix \ref{sec:AppendixA}]. 
\end{theorem}
Theorem \ref{thm:grad_Approximation} upper bounds the distance between the learning step on the graph and on the graphon {for a given input-output pair $(X,Y)\in\ccalD$} as a function of the size of the graph. Note that this bound is controlled by two terms.
The first term decreases with $n$ and is an \textit{approximation bound} related to the approximation of $\bbW$ by $\bbS_n$. The second term is a constant \textit{nontransferable bound} controlled by $c$ and related to a threshold that we impose on the convolutional filter $h$. Since graphons $\bbW$ have an infinite spectrum that accumulates around zero, we can only show convergence of all of the spectral components of $Y_n$ in the limit \cite{ruiz2020graphonTransferability}. Given that $n$ is finite, we can only upper bound distances between spectral components associated with eigenvalues larger than some fixed threshold $c \in (0,1]$. The perturbations associated with the other spectral components are controlled by bounding the variation of the filter response below $c$.  

The bound in Theorem \ref{thm:grad_Approximation} is important because it allows us to quantify the error incurred by taking gradient steps not on the graphon, but on the graph data. This is instrumental for learning a meaningful solution of the graphon ERM problem \eqref{eqn:ERM}. Although we cannot calculate gradients of the function that we want to learn [cf. \eqref{eqn:WNN_Learning_Step}], Theorem \ref{thm:grad_Approximation} shows that we can expect to follow the direction of the gradient on the WNN by measuring the ratio between the norm of the gradient of the loss on the GNN and the difference between the two gradients.
Note however that we are only able to decrease this approximation bound down to the value of the nontransferable bound $\gamma c$. 


\begin{remark}[Feasibility of assumptions] 
AS\ref{as1}--AS\ref{as4} are normalized Lipschitz smoothness conditions which can be relaxed by making the Lipschitz constant greater than one. AS\ref{as3} holds for most activation functions used in practice, e.g., the ReLU and the sigmoid . AS\ref{as4} can be achieved by normalization and holds for most loss functions in a closed set (e.g., the hinge or mean square losses). AS\ref{as5} is necessary to guarantee a $\ccalO(\sqrt{\log{n}/n})$ rate of convergence of $\bbS_n$ to $\bbW$ \cite{chung2011spectra}.
\end{remark}


\section{Training on Growing Graphs}



An insightful observation that can be made from Theorem \ref{thm:grad_Approximation} is that, since the discretization error depends on the number of nodes of the graph, it is possible to reduce its effect by iteratively increasing the graph size at regular intervals during the training process. In Algorithm \ref{alg:WNNL}, we propose a training procedure that does so at every epoch.





\begin{algorithm}[t]
		\caption{Learning by transference}
		\label{alg:WNNL}
		\begin{algorithmic}[1]
			\State   Initialize $\ccalH_0,n_0$ and sample graph $\bbG_{n_0}$ from graphon $\bbW$
			\Repeat \ for epochs $0,1,\ldots$
			\For {$k$ =1,\dots, $|\ccalD|$}
			\State Obtain sample $(Y,X)\sim \ccalD$
			\State Construct graph signal $\bby_n,\bbx_n$ [cf. \eqref{eqn:gcn_obtained}]
			\State Take learning step:\\ \quad\quad\quad\quad$\ccalH_{k+1}=\ccalH_{k}-\eta_k\nabla \ell(\bby_n,\bbPhi(\bbx_n;\ccalH_k,\bbS_n))$
			\EndFor
			\State {Increase number of nodes $n$}
			\State {Sample $\bbS_n$ Bernoulli from graphon $\bbW$ [cf. \eqref{eqn:BernoulliSampledGraph}]}
			\Until  convergence
		\end{algorithmic}
\end{algorithm}

In practice, in the ERM problem the number of nodes that can be added is upper bounded by the maximum graph size in the dataset. As long as we stay under this limit, we can decide which and how many nodes to consider for training arbitrarily. The novelty of Algorithm \ref{alg:WNNL} is that it does not require using the largest available graph to train the GNN. Instead, we can set a minimum graph size $n_0$ and progressively increase it up to the total number of nodes. {The main advantage of Algorithm \ref{alg:WNNL} is thus that it allows reducing the computational cost of training the GNN without compromising its performance. 

{A bound on the difference between the gradients does not, however, imply that we can follow the gradient on the graph blindly. In other words, the fact that gradients are close does not imply that the direction that the iterates follow in the graphon and the graph is the same. In the following subsection, we will thus use Theorem \ref{thm:grad_Approximation} to obtain the conditions under which the gradient steps taken on the graph follow the learning direction we would expect to follow on the graphon, so that Algorithm \ref{alg:WNNL} converges to a neighborhood of the optimal solution on the graphon.

\subsection{Algorithm Convergence}

{To prove that Algorithm \ref{alg:WNNL} converges, we need the following Lipschitz assumption on the neural network and its gradients.}

\begin{assumption} \label{as7}
	The graphon neural network $ \bbPhi(X;\ccalH,\bbW)$ is $\AWNN$-Lipschitz, and its gradient $ \nabla_\ccalH\bbPhi(X;\ccalH,\bbW)$ is $\AgWNN$-Lipschitz, with respect to the parameters $\ccalH$ [cf. Definition \ref{def:norm_lips}].
\end{assumption}

We also need the following lemma that we prove in Appendix \ref{sec:AppendixB}.
\begin{lemma}\label{lemma:martingale}
		Consider the ERM problem in \eqref{eqn:ERM} and let $\bbPhi(X;\ccalH,\bbW)$ be an $L$-layer WNN with $F_0=F_L=1$, and $F_{l}=F$ for $1\leq l \leq L-1$. Let $c \in (0,1]$, $\epsilon>0$, step size $\eta<{\Anl^{-1}}$, with $\Anl=\AnWNN+\AWNN F^{2L} \sqrt{K}$  and assume that the graphon convolutions in all layers of this WNN have $K$ filter taps [cf. \eqref{eq:graphon_convolution}]. Let $\bbPhi(\bbx_n;\ccalH,\bbS_n)$ be a GNN sampled from $\bbPhi(X;\ccalH,\bbW)$ as in \eqref{eqn:gcn_obtained}. Consider the iterates generated by equation \eqref{eqn:GNN_Learning_Step}. Under Assumptions AS\ref{as1}-AS\ref{as7}, if at step $k$ the number of nodes $n$ verifies
		\begin{align}
	      \mbE[\|\nabla_{\ccalH_{k}}&\ell(Y,\bbPhi(X;\ccalH_k,\bbW))-\nabla_{\ccalH}\ell(Y_n,\bbPhi(X_n;\ccalH,\bbW_n))\|]  \nonumber\\
	      & \leq	\|\nabla_{\ccalH} \ell(Y,\bbPhi(X;\ccalH_k,\bbW))\|
		\end{align}
        then the iterate generated by graph learning step \eqref{eqn:GNN_Learning_Step} satisfies
        \bal
        \mbE[\ell(Y,\bbPhi(X;\ccalH_{k+1},\bbW))]\leq\ell(Y,\bbPhi(X;\ccalH_{k},\bbW)).
        \eal
    \end{lemma}
	
	{Lemma \ref{lemma:martingale} states that if the norm of the difference between the gradients on the graph and the graphon is bounded by the norm of the gradient on the graphon, then the loss on the graphon can be decreased by taking steps on the graph. In conjuction with Theorem \ref{thm:grad_Approximation}, Lemma \ref{lemma:martingale} thus gives an upper bound for the number of nodes that needed to follow the learning direction on the graphon, implying that it is always possible to construct a decreasing sequence of loss values by taking steps on the sampled graphs. 
	
	The specific condition on the number of nodes is determined by the norm of the gradient of the loss on the graphon. The intuition behind this is that, as the GNN improves its performance, the value of the loss decreases, and so the magnitude of its gradient decreases as well. Alternatively, we can think of the rate of decrease of the graphon gradient as a measure of the difficulty of the the learning problem, to which Lemma \ref{lemma:martingale} allows us to set the number of nodes accordingly.
	
	
	Our main result is stated in Theorem \ref{thm:WNN_learning}, which shows that Algorithm \ref{alg:WNNL} converges if the rate at which the graph grows satisfies the condition in \eqref{eqn:thm2_condition}.

	\begin{theorem}\label{thm:WNN_learning}
		Consider the ERM problem in \eqref{eqn:ERM} and let $\bbPhi(X;\ccalH,\bbW)$ be an $L$-layer WNN with $F_0=F_L=1$, and $F_{l}=F$ for $1\leq l \leq L-1$.  Let $c \in (0,1]$, $\epsilon>0$, step size $\eta<{\Anl}^{-1}$, with $\Anl=\AnWNN+\AWNN F^{2L} \sqrt{K}$ and assume that the graphon convolutions in all layers of this WNN have $K$ filter taps [cf. \eqref{eq:graphon_convolution}]. Let $\bbPhi(\bbx_n;\ccalH,\bbS_n)$ be a GNN sampled from $\bbPhi(X;\ccalH,\bbW)$ as in \eqref{eqn:gcn_obtained}. Consider the iterates generated by equation \eqref{eqn:GNN_Learning_Step}, under Assumptions AS\ref{as1}-AS\ref{as7}, if at each step $k$ the number of nodes $n$ verifies 
    \begin{align}\label{eqn:thm2_condition}
	      \mbE[\|\nabla_{\ccalH_{k}}&\ell(Y,\bbPhi(X;\ccalH_k,\bbW))-\nabla_{\ccalH}\ell(Y_n,\bbPhi(X_n;\ccalH,\bbW_n))\|]  \nonumber\\
	      & +\epsilon < 	\|\nabla_{\ccalH} \ell(Y,\bbPhi(X;\ccalH_k,\bbW))\|
	\end{align}
    then Algorithm \ref{alg:WNNL} converges to an $\epsilon$-neighborhood of the solution of the Graphon Learning problem \eqref{eqn:ERM} in at most $k^*=\ccalO(1/\epsilon^2)$ iterations, where $\gamma$ is a constant that depends on the number of layers $L$, features $F$, and filter taps $K$ of the GNN [cf. Definition \ref{def:Constant_Of_GNN} in the Appendix \ref{sec:AppendixA}].
	\end{theorem}
\begin{proof} For every $\epsilon>0$, we define the stopping time $k^*$ as
\begin{align}
    k^*:=\min_{k\geq 0}\{\mbE_\ccalD[\|\nabla_{\ccalH} \ell(Y,\bbPhi(X;\ccalH_k,\bbW))\|]\leq  \gamma c+\epsilon \}.
\end{align}
Given the iterates at $k=k^*$ and the initial values at $k=0$, we can express the expected difference between the loss $\ell$ as the summation over the difference of iterates,
\begin{align}
    &\mbE[\ell(Y,\bbPhi(X;\ccalH_{0},\bbW))-\ell(Y,\bbPhi(X;\ccalH_{k^*},\bbW))]=\nonumber\\
    &\mbE\left[\sum_{k=1}^{k^*}\ell(Y,\bbPhi(X;\ccalH_{k-1},\bbW))-\ell(Y,\bbPhi(X;\ccalH_{k},\bbW)\right] \text{.} \nonumber
\end{align}
    Taking the expected value with respect to the final iterate $k=k^*$, we get
\begin{align}
    &\mbE\bigg[\ell(Y,\bbPhi(X;\ccalH_{k^0},\bbW))-\ell(Y,\bbPhi(X;\ccalH_{k^*},\bbW))\bigg]\nonumber\\
    &=\mathop{\mbE}_{k^*}\bigg[\mbE\bigg[\sum_{k=1}^{k^*}\ell(Y,\bbPhi(X;\ccalH_{k-1},\bbW))-\ell(Y,\bbPhi(X;\ccalH_{k},\bbW)\bigg]\bigg]\nonumber\\
    &=\sum_{t=0}^\infty \mbE\bigg[\sum_{k=1}^{t}\ell(Y,\bbPhi(X;\ccalH_{k-1},\bbW))\nonumber\\
    &\quad\quad\quad-\ell(Y,\bbPhi(X;\ccalH_{k},\bbW)\bigg]P(k^*=t) \text{.}\label{eqn:lemma6_law_of_total_probability}
\end{align}
Lemma \ref{lemma:martingale} applied to any $k\leq k^*$ verifies
\begin{align}
    &\mbE\bigg[\ell(Y,\bbPhi(X;\ccalH_{k-1},\bbW))-\ell(Y,\bbPhi(X;\ccalH_{k},\bbW))\bigg]\nonumber\geq \eta \gamma \epsilon^2 \text{.}
\end{align}
Coming back to \eqref{eqn:lemma6_law_of_total_probability}, we get
\begin{align}
    \mbE\bigg[\ell(Y,\bbPhi(X;\ccalH_{k^0},&\bbW))-\ell(Y,\bbPhi(X;\ccalH_{k^*},\bbW))\bigg]\nonumber\\
    &\geq \eta\gamma \epsilon^2 \sum_{t=0}^\infty t P(k^*=t)= \eta \gamma \epsilon^2 \mbE[k^*] \text{.}\nonumber
\end{align}
Since the loss function $\ell$ is non-negative, 
\begin{align}
    \frac{\mbE\bigg[\ell(Y,\bbPhi(X;\ccalH_{k^0},\bbW))\bigg]}{ \eta \gamma \epsilon^2 }
    &\geq  \mbE[k^*], \nonumber
\end{align}
from which we conclude that $k^* =\ccalO(1/\epsilon^2)$.
\end{proof}

\begin{figure*}[t]
     \centering
     \begin{subfigure}[b]{0.45\textwidth}
         \centering
         \includegraphics[width=\textwidth]{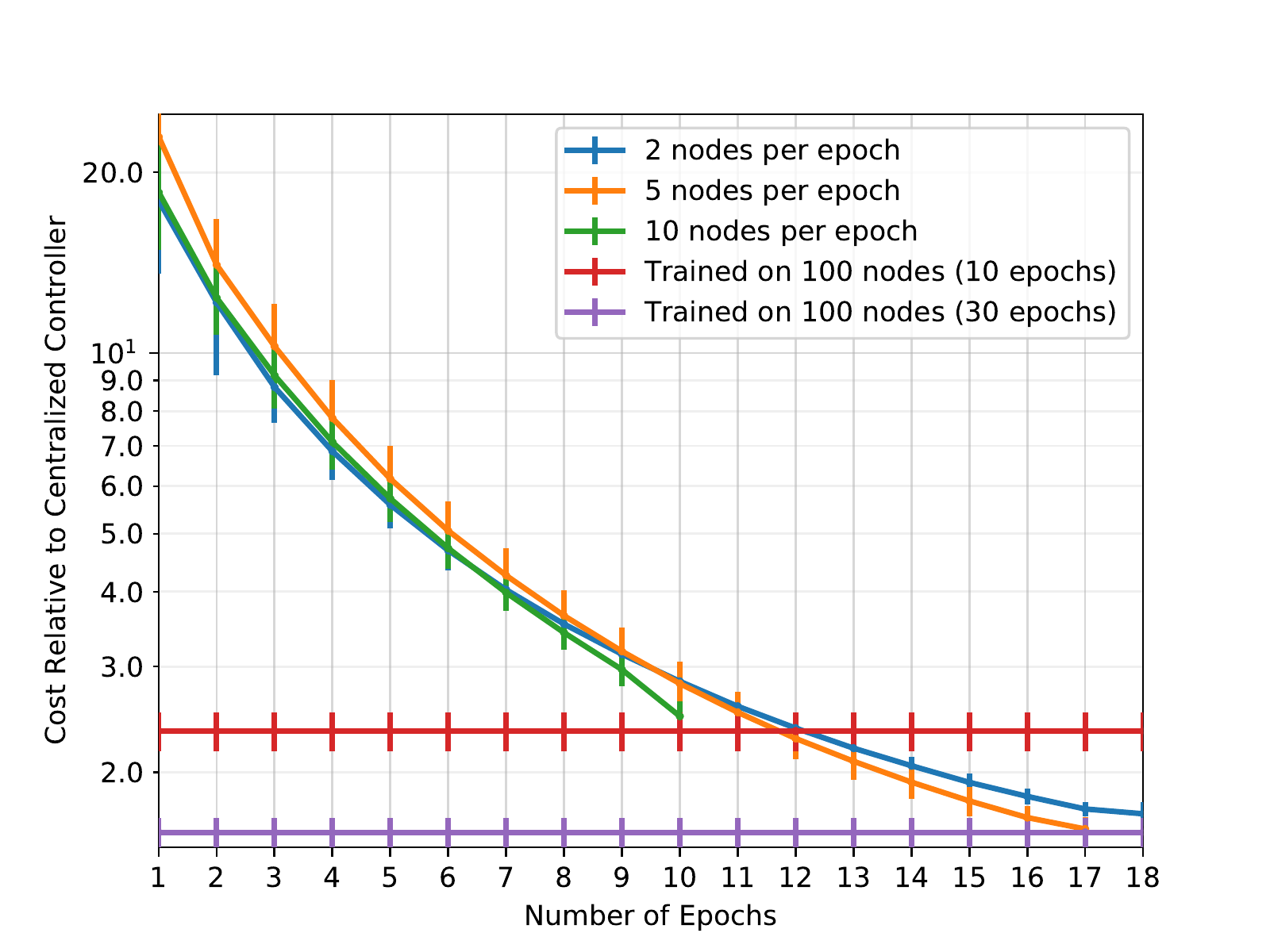}
         \caption{Starting with $10$ nodes}
     \end{subfigure}
     \hfill
     \begin{subfigure}[b]{0.45\textwidth}
         \centering
         \includegraphics[width=\textwidth]{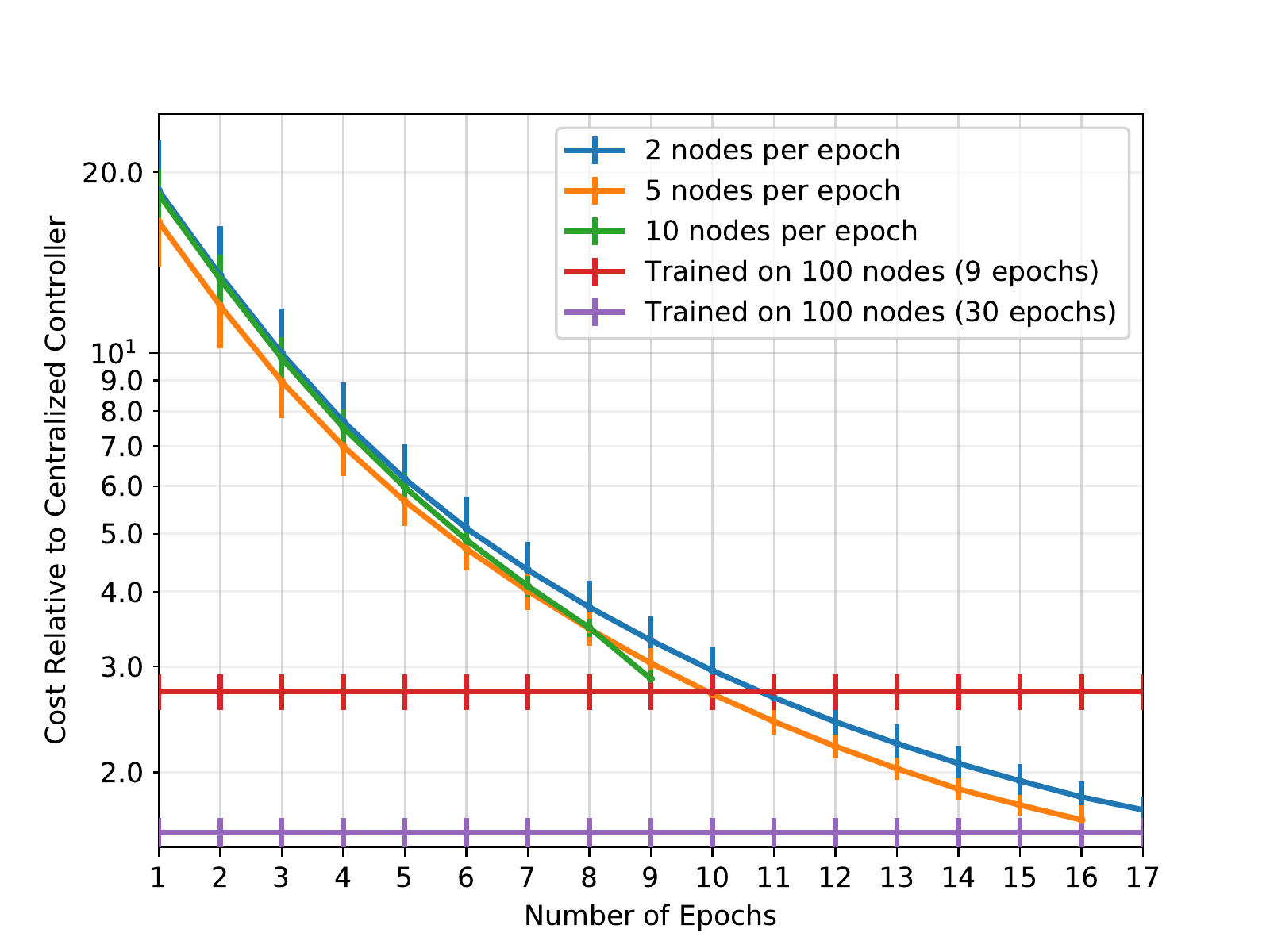}
         \caption{Starting with $20$ nodes}
     \end{subfigure}
	\caption{Velocity variation of the flocking problem for the whole trajectory in the testing set relative to the centralized controller.}
	\label{fig:flocking}
\end{figure*}

The condition in \eqref{eqn:thm2_condition} implies that the rate of decrease in the norm of the gradient of the loss on the GNN
needs to be slower than roughly $n^{-1/2}$. If the norm of the gradient does not vary between iterations, the number of nodes does not need to be increased.
We also see that the convergence of the gradient on the graphon given by Theorem \ref{thm:grad_Approximation} is equal to the bias term in \eqref{eq:thm_1} plus a constant $\epsilon$. We can keep increasing the number of nodes---and thus decreasing the bias---until the norm of the GNN gradient is smaller than the nontransferable constant value. Once this point is attained, there is no gain in decreasing the approximation bound any further \cite{ajalloeian2020analysis}. Recall that the constant term can be made as small as desired by tuning $c$, at the cost of decreasing the approximation bound. Further, note that assuming smoothness of the GNN is a mild assumption \cite{scaman2018lipschitz,jordan2020exactly,latorre2020lipschitz,fazlyab2019efficient,tanielian2021approximating,du2019gradient}. A characterization of the Lipschitz constant is out of the scope of this work. 



\section{Numerical Results}
\label{sec:Numericals}




In this section we consider the problem of coordinating a set of $n$ agents initially flying at random to avoid collisions, and to fly at the same velocity. Also known as flocking, at each time $t$ agent $i$ knows its own position $r_i(t)\in \reals^2$, and speed $v_i(t)\in\reals^2$, and reciprocally exchanges it with its neighboring nodes if a communication link exists between them. Links are govern by physical proximity between agents forming a time varying graph $\bbG_n=(\ccalV,\ccalE)$. A communication link exists if the distance between two agents $i,j$ satisfies 
\begin{align}
    r_{ij}(t)=\|r_i(t)-r_j(t)\|\leq R=2m.
\end{align}
We assume that at each time $t$ the controller sets an acceleration $u_i \in [-10,10]^2$, and that it remains constant for a time interval $T_s=20ms$. The system dynamics are govern by,
\begin{align}
    r_i(t+1)&=u_i(t)T_s^2/2+v_i(t)T_s+r_i(t)\\
    v_i(t+1)&=u_i(t)T_s+v_i(t).
\end{align}
To avoid the swarm of robots to reach a null common velocity, we initialize the velocities at random $\bbv(t)=[v_1(t),\dots,v_n(t)]$, by uniformly sampling a common bias velocity $v_{BIAS}\sim \ccalU[-3,3]$, and then adding independent uniform noise $\ccalU[-3,3]$ to each agent.  The initial deployment is randomly selected in a circle always verifying that the minimum distance between two agents is larger than $0.1m$. On the one hand, agents pursue a common average velocity $\bar\bbv=(1/n)\sum_{i=1}^n v_i(t)$, thus minimizing the velocity variation of the team. On the other hand, agents are required to avoid collision. We can thus define the velocity variation of the team
\begin{align}
\sigma_{\bbv(t)}=\sum_{i=1}^{n}\|v_i(t)-\bar \bbv(t)\|^2,
\end{align}
and the collision avoidance potential 
\begin{align*}
\begin{split}
CA_{ij}=\begin{cases}
	\frac{1}{\|r_i-r_j\|^2}-\log(\|r_i-r_j\|^2) & \text{if $\|r_i-r_j\| \leq R_{CA}$}\\
	\frac{1}{R_{CA}^2}-\log(R_{CA}^2) & \text{otherwise,}
\end{cases} 
\end{split}
\end{align*}
with $R_{CA}=1m$. A centralized controller can be obtain by $u_i(t)^*=-n(v_i-\bar \bbv)+\sum_{j=1}^n\nabla_{r_i}CA(r_i,r_j)$ \cite{tanner2003stable}.

Exploiting the fact that neural networks are universal approximators \cite{barron1993universal,hornik1991approximation}, \textit{Imitation learning} can be utilized as a framework to train neural networks from information provided by an expert \cite{ross2011reduction,ross2010efficient}. Formally, we have a set of pairs $\ccalT= \{ \bbx_{m},\bbu^*_{m} \},m=1,\dots,M$, and during training we minimize the mean square error between the optimal centralized controller, and the output of our GNN $\|\bbu^*_{m}-\bbPhi(\bbx_m;\ccalH,\bbS)\|^2$. Denoting $\ccalN_i(t)$ the neighborhood of agent $i$ at time $t$, the state of the agents $\bbx(t)=[x(t)_{1},\dots,x(t)_{n}], x_i(t)\in \reals^6$, is given by
$$
x_i(t)=\sum_{j:j\in \ccalN_i(t)}\bigg[v_i(t)-v_j(t),\frac{r_{ij}(t)}{\|r_{ij}(t)\|^4},\frac{r_{ij}(t)}{\|r_{ij}(t)\|^2}\bigg]\text{.}
$$
Note that state $x_i(t)$, gets transmitted between agents if a communication link exists between them. 

In Figure \ref{fig:flocking} we can see the empirical manifestation of the claims we put forward. First and foremost, we are able to learn a GNN that achieves a comparable performance while taking steps on a smaller graphs. As seen in Figure \ref{fig:flocking}, GNNs trained with $n_0=\{10,20\}$ agents in the first epoch and adding $10$ agents per epoch (green line) are able to achieve a similar performance when reaching $100$ agents than the one they would have achieved by training with $100$ agents the same number of epochs. Besides, if we add less nodes per epoch, we are able to achieve a similar performance that we would have achieved by training on the large network for $30$ epochs. 

\section{Conclusions}
\label{sec:Conclusions}
We have introduced a learning procedure for GNNs that progressively grows the size of the graph while training. Our algorithm requires less computational cost---as the number of nodes in the graph convolution is smaller---than training on the full graph without compromising performance. Leveraging transferability results, we bounded the expected difference between the gradient on the GNN, and the gradient on the WNN. {Utilizing this result, we provided the theoretical guarantees that our Algorithm converges to a neighborhood of a first order stationary point of the WNN in finite time.} We benchmarked our algorithm on a recommendation system and a decentralized control problem, achieving comparable performance to the one achieve by a GNN trained on the full graph.

\appendices

\section{Proof of Theorem \ref{thm:grad_Approximation}}
\label{sec:AppendixA}

\begin{proposition} \label{prop:x-x'}
	Let $X \in L_2([0,1])$ be a normalized Lipschitz graphon signal, and let $X_n$ be the graphon signal induced by the graph signal $\bbx_n$ obtained from $X$ on the template graph $\mbG_n$, i.e., $[\bbx_n]_i = X((i-1)/n)$ for $1 \leq i \leq n$. It holds that
	\begin{align}
	\|X-X_n\| \leq \dfrac{{1}}{n} .
	\end{align}
\end{proposition}
\begin{proof}
	Let $I_i = [(i-1)/n, i/n)$ for $1 \leq i \leq  n-1$ and $I_n = [(n-1)/n,1]$. 
	Since the graphon is normalized Lipschitz, for any $u \in I_i$, $1 \leq i \leq n$, we have
	\begin{align*}
	\|X(u) - X_n(u)\| \leq \max\left(\left|u-\frac{i-1}{n}\right|,\left|\frac{i}{n}-u\right|\right) \leq \frac{1}{n}\nonumber.
	\end{align*}
	We can then write
	\begin{align*}
	\|X-X_n\|^2 &= \int_0^1 |X(u) - X_n(u)|^2 du \\
	&\leq \int_0^1 \left(\frac{1}{n}\right)^2 du = \left(\frac{1}{n}\right)^2,
	\end{align*}
	which completes the proof.
\end{proof}

\begin{proposition} \label{prop:w-w'}
	Let $\bbW: [0,1]^2 \to [0,1]$ be a normalized Lipschitz graphon, and let $\mbW_n := \bbW_{\mbG_n}$ be the graphon induced by the template graph $\mbG_n$ generated from $\bbW$. It holds that
	\begin{equation}
	\|\bbW-\mbW_n\| \leq \dfrac{2}{n} \text{.}
	\end{equation}
\end{proposition}
\begin{proof}
	Let $I_i = [(i-1)/n, i/n)$ for $1 \leq i \leq  n-1$ and $I_n = [(n-1)/n,1]$. 
	Since the graphon is Lipschitz, for any $u \in I_i$, $v \in I_j$, $1 \leq i,j \leq n$, we have
	\begin{align*}
	\|\bbW(u,v)-\mbW_n(u,v)\| &\leq \max\left(\left|u-\frac{i-1}{n}\right|,\left|\frac{i}{n}-u\right|\right) \nonumber\\
	&+ \max\left(\left|v-\frac{j-1}{n}\right|,\left|\frac{j}{n}-v\right|\right) \nonumber\\
	&\leq \frac{1}{n} + \frac{1}{n} = \frac{2}{n}\text{.}
	\end{align*}
	We can then write
	\begin{align*}
	\|\bbW-\mbW_n\|^2 &= \int_0^1 |\bbW(u,v)-\mbW_n(u,v)|^2 du dv \\
	&\leq \int_0^1 \left(\frac{2}{n}\right)^2 du dv = \left(\frac{2}{n}\right)^2
	\end{align*}
	which concludes the proof.
\end{proof}

\begin{lemma}\label{lemma:norm_grad_on_wnn_bound}
	Consider the $L$-layer WNN given by $Y=\bbPhi(X; \ccalH, \bbW)$, where $F_0=F_L=1$ and $F_\ell=F$ for $1 \leq \ell \leq L-1$. Let $c\in (0,1]$ and assume that the graphon convolutions in all layers of this WNN have $K$ filter taps [cf. \eqref{eq:graphon_convolution}]. Under Assumptions \ref{as1} through \ref{as3}, the norm of the gradient of the WNN with respect to its parameters $\ccalH=\{\bbH_{lk}\}_{l,k}$ can be upper bounded by, 
	\bal
	\|\nabla_{\ccalH}\bbPhi(X;\ccalH,\bbW)\| \leq F^{2L} \sqrt{K}.
	\eal
\end{lemma}
\begin{proof}
	We will find an upper bound for any element $\Hdag$ of the tensor $\ccalH$. We start by the last layer of the WNN, applying the definition given in \eqref{eqn:wcn_layer}
	\bal
	&\|\nabla_{\Hdag}\bbPhi(X;\ccalH,\bbW)\|=\bigg\|\nabla_\Hdag X^f_{L}\bigg\| \nonumber \\
	&=\bigg\|\nabla_\Hdag \rho\left(\sum_{g=1}^{F_{l-1}} \sum_{k=1}^{K-1} (T_{\bbW}^{(k)} X^g_{l-1})[\bbH_{Lk}]_{gf} \right)\bigg\|. \nonumber
	\eal
	By Assumption \ref{as3}, the nonlinearity $\rho$ is normalized Lipschitz, i.e., $\nabla \rho(u)\leq 1 $ for all $u$. Thus, applying the chain rule for the derivative and the Cauchy-Schwartz inequality, the right hand side can be rewritten as
	\bal
	&\|\nabla_{\Hdag}\bbPhi(X;\ccalH,\bbW)\|\nonumber\\
	&=\bigg\|\nabla  \rho\left(\sum_{g=1}^{F_{l-1}} \sum_{k=1}^{K-1} (T_{\bbW}^{(k)} X^g_{l-1})[\bbH_{Lk}]_{gf} \right) \bigg\|\nonumber\\
	&\times \bigg\| \nabla_\Hdag \sum_{g=1}^{F_{l-1}} \sum_{k=1}^{K-1} (T_{\bbW}^{(k)} X^g_{l-1})[\bbH_{Lk}]_{gf}\bigg\|\nonumber\\
	&\leq\bigg\| \nabla_\Hdag \sum_{g=1}^{F_{l-1}} \sum_{k=1}^{K-1} (T_{\bbW}^{(k)} X^g_{l-1})[\bbH_{Lk}]_{gf}\bigg\| \text{.}\nonumber
	\eal
	Note that the a larger bound will occur if $l^\dagger < L-1$, so by linearity of derivation and by the triangle inequality, we obtain
	\bal
	\|&\nabla_{\Hdag}\bbPhi(X;\ccalH,\bbW)\|\nonumber \\
	&\leq \sum_{g=1}^{F_{l-1}}\bigg\|   \sum_{k=1}^{K-1} T_{\bbW}^{(k)} (\nabla_\Hdag X^g_{l-1})[\bbH_{Lk}]_{gf}\bigg\| \text{.} \nonumber
	\eal
	By Assumption \ref{as2}, the convolutional filters are non-amplifying, thus it holds that
	\bal
	\|\nabla_{\Hdag}\bbPhi(X;\ccalH,\bbW)\|\leq&\sum_{g=1}^{F_{l-1}}\bigg\| \nabla_\Hdag X^g_{l-1}\bigg\|\text{.} \nonumber
	\eal
	Now note that as filters are non-amplifying, the maximum difference in the gradient will be attained at the first layer ($l=1$) of the WNN. Also note that the derivative of a convolutional filter $T_\bbH$ at coefficient $k^{\dagger}=i$ is itself a convolutional filter with coefficients $\bbh_i$. The values of $\bbh_i$ are $[\bbh_i]_j=1$ if $j=i$ and $0$ otherwise. Thence, 
	\bal
	\|\nabla_{\Hdag}\bbPhi(X;\ccalH,\bbW)\|
	&\leq  F^{L-1} \bigg\| {\bbh_i}_{*\bbW}X_0 \bigg\| \nonumber\\
	&\leq  F^{L-1} \| X_0 \|.\label{eqn:norm_grad_individual_bound} 
	\eal
	To complete the proof note that tensor $\ccalH$ has $F^{L-1}K$ elements, and each individual gradient is upper bounded by \eqref{eqn:norm_grad_individual_bound}, and $\|X\|$ is normalized by Assumption \ref{as1}.		
\end{proof}

\begin{lemma}\label{lemma1:WNN-WNNn}
    Let $\bbPhi(X;\ccalH,\bbW)$ be a WNN with $F_0=F_L=1$, and $F_l=F$ for $1\leq l \leq L-1$. Let $c \in (0,1]$, and assume that the graphon convolutions in all layers of this WNN have $K$ filter taps [cf. \eqref{eq:graphon_convolution}]. Let $\bbPhi(\bbx_n;\ccalH,\bbS_n)$ be a GNN sampled from $\bbPhi(X;\ccalH,\bbW)$ as in \eqref{eqn:gcn_obtained}. Under Assumptions \ref{as1},\ref{as2},\ref{as3} and \ref{as5}, with probability $1-\xi$ it holds that
    \begin{align}
        \|\bbPhi(X;&\ccalH,\bbW)-\bbPhi(X_n;\ccalH,\bbW_n)\|\nonumber\\
        &\leq LF^{L-1}\bigg(1+\frac{\pi B_{\bbW_n}^c}{\delta_{\bbW\bbW_n}^c} \bigg)\frac{2\big(1+\sqrt{n \log({2n}/{\xi})} \big)}{n}\nonumber\\
        &+\frac{1}{n}+4LF^{L-1}c \text{.}
    \end{align}
The fixed constants $B^c_{\bbW}$ and $\delta^c_{\bbW \bbW_n}$ are the $c$-band cardinality and the $c$-eigenvalue margin of $\bbW$ and $\bbW_n$ respectively [cf. Definitions \ref{def:c_band_cardinality},\ref{def:c_eigenvalue_margin}].
\end{lemma}
\begin{proof}
The proof can be found in \cite[Theorem $3$]{ruiz2021transferability} where $A_h=1$, $A_w=1$, and $\|X\|=1$ by Assumptions \ref{as1}, \ref{as2}; and $\alpha(n,\chi)=1$ and $\beta(n,\chi_3)=\sqrt{n \log({2n}/{\xi})}$.
\end{proof}

\begin{lemma}\label{lemma:norm_bound_dif_grad_WNN}
	Let $\bbPhi(X;\ccalH,\bbW)$ be a WNN with $F_0=F_L=1$, and $F_l=F$ for $1\leq l \leq L-1$. Let $c \in (0,1]$, and assume that the graphon convolutions in all layers of this WNN have $K$ filter taps [cf. \eqref{eq:graphon_convolution}]. Let $\bbPhi(\bbx_n;\ccalH,\bbS_n)$ be a GNN sampled from $\bbPhi(X;\ccalH,\bbW)$ as in \eqref{eqn:gcn_obtained}. Under Assumptions \ref{as1},\ref{as2},\ref{as3} and \ref{as5}, with probability $1-\xi$ it holds that
	\bal
	&\|\nabla_{\ccalH}\bbPhi(X;\ccalH,\bbW)-\nabla_{\ccalH}\bbPhi(X_n;\ccalH,\bbW_n)\|\\
	&\leq  \sqrt{KF^{L-1}}\bigg(\frac{2 F^{L-1}L}{n}+8L^2F^{2L-2}c \nonumber\\
    &+2 L^2F^{2L-2}\bigg(1+\frac{\pi B_{\bbW_n}^c}{\delta_{\bbW\bbW_n}^c} \bigg)\frac{2\big(1+\sqrt{n \log({2n}/{\xi})} \big)}{n}\bigg). \nonumber
	\eal
\end{lemma}
\begin{proof}
	We will first show that the gradient with respect to any arbitrary element $\Hdag \in \reals$ of $\ccalH$ can be uniformly bounded. Note that the maximum is attained if $l^\dagger=1$. Without loss of generality, assuming $l^\dagger>l-1$, we can begin by using the output of the WNN in \eqref{eqn:wcn_layer} to write
	\bal
	\begin{split}
	\|&\nabla_{\Hdag}\bbPhi(X;\ccalH,\bbW)-\nabla_{\Hdag}\bbPhi(X_n;\ccalH,\bbW_n)\|\\\
	&\quad\quad=\|\nabla_{\Hdag}X_{L}^f-\nabla_{\Hdag}X_{nL}^f\|\\
	&\quad\quad=\bigg\|\nabla_\Hdag \rho\left(\sum_{g=1}^{F_{l-1}} \sum_{k=1}^{K-1} (T_{\bbW}^{(k)} X^g_{l-1})[\bbH_{lk}]_{gf} \right)\\
	&\quad\quad-\nabla_\Hdag \rho\left(\sum_{g=1}^{F_{l-1}} \sum_{k=1}^{K-1} (T_{\bbW_n}^{(k)} X^g_{n l-1})[\bbH_{lk}]_{gf} \right)\bigg\|.
	\end{split}
	\label{eq:norm_bound_dif__grad_WNN_first_equality}
	\eal
	
	Applying the chain rule and using the triangle inequality, we get
	\bal
	\|&\nabla_{\Hdag}X_{L}^f-\nabla_{\Hdag}X_{nL}^f\|\nonumber\\
	&\quad\quad\leq \bigg\|\bigg(\nabla \rho \bigg(\sum_{g=1}^{F_{l-1}} \sum_{k=1}^{K-1} (T_{\bbW}^{(k)} X^g_{ l-1})[\bbH_{lk}]_{gf}\bigg)\nonumber\\
	&\quad\quad-\nabla \rho \bigg(\sum_{g=1}^{F_{l-1}} \sum_{k=1}^{K-1} (T_{\bbW_n}^{(k)} X^g_{n l-1})[\bbH_{lk}]_{gf}\bigg)\bigg)\nonumber\\	
	&\quad\quad\times\nabla_{\Hdag}\bigg(\sum_{g=1}^{F_{l-1}} \sum_{k=1}^{K-1} (T_{\bbW}^{(k)} X^g_{ l-1})[\bbH_{lk}]_{gf}\bigg)\bigg\|\nonumber\\	
	&\quad\quad+\bigg\|\nabla \rho \bigg(\sum_{g=1}^{F_{l-1}} \sum_{k=1}^{K-1} (T_{\bbW_n}^{(k)} X^g_{n l-1})[\bbH_{lk}]_{gf}\bigg)\nonumber\\	
	&\quad\quad\times\bigg({\tiny\nabla_\Hdag}  \sum_{g=1}^{F_{l-1}} \sum_{k=1}^{K-1} (T_{\bbW}^{(k)} X^g_{ l-1})[\bbH_{lk}]_{gf}\nonumber\\
	&\quad\quad-{\tiny\nabla_\Hdag}   \sum_{g=1}^{F_{l-1}} \sum_{k=1}^{K-1} (T_{\bbW_n}^{(k)} X^g_{n l-1})[\bbH_{lk}]_{gf}\bigg)\bigg\|\nonumber.
	\eal
	
	Next, we use Cauchy-Schwartz inequality, Assumptions \ref{as3}, \ref{as4}, and Proposition \ref{lemma:norm_grad_on_wnn_bound} to bound the terms corresponding to the gradient of the nonlinearity $\rho$, the loss function $\ell$, and the WNN respectively. Explicitly,
	\bal
	\|&\nabla_{\Hdag}X_{L}^f-\nabla_{\Hdag}X_{nL}^f\|\label{eqn:prop_norm_bound_dif__grad_WNN_divided_eqns}\\
	&\quad\quad\quad\leq \bigg\|\sum_{g=1}^{F_{l-1}} \sum_{k=1}^{K-1} (T_{\bbW}^{(k)} X^g_{ l-1})[\bbH_{lk}]_{gf}\nonumber\\
	&\quad\quad\quad-\sum_{g=1}^{F_{l-1}} \sum_{k=1}^{K-1} (T_{\bbW_n}^{(k)} X^g_{n l-1})[\bbH_{lk}]_{gf}\bigg\|F^{L-1}\|X_0\|\nonumber\\	
	&\quad\quad\quad+\bigg\|\sum_{g=1}^{F_{l-1}} \nabla_\Hdag \sum_{k=1}^{K-1}  \bigg(  (T_{\bbW}^{(k)} X^g_{ l-1})[\bbH_{lk}]_{gf}\nonumber\\
	&\quad\quad\quad-(T_{\bbW_n}^{(k)} X^g_{n l-1})[\bbH_{lk}]_{gf})\bigg)\bigg\|\nonumber .
	\eal
    Applying triangle inequality to the second term, we get
	\bal
	\|&\nabla_{\Hdag}X_{L}^f-\nabla_{\Hdag}X_{nL}^f\|\nonumber\\
	&\quad\leq \bigg\|\sum_{g=1}^{F_{l-1}} \sum_{k=1}^{K-1} (T_{\bbW}^{(k)} X^g_{ l-1})[\bbH_{lk}]_{gf}\nonumber\\	
	&\quad-\sum_{g=1}^{F_{l-1}} \sum_{k=1}^{K-1} (T_{\bbW_n}^{(k)} X^g_{n l-1})[\bbH_{lk}]_{gf}\bigg\| F^{L-1}\|X_0\|\nonumber\\	
	&\quad+\bigg\|\sum_{g=1}^{F_{l-1}} {\tiny\nabla_\Hdag \sum_{k=1}^{K-1}  \bigg(  (T_{\bbW}^{(k)})} [\bbH_{lk}]_{gf}\nonumber\\	
	&\quad-(T_{\bbW_n}^{(k)} )[\bbH_{lk}]_{gf})\bigg)X^g_{n l-1} \bigg\|\label{eqn:norm_bound_dif_grad_WNN_long_triangle_inequality} \\
	&\quad+\sum_{g=1}^{F_{l-1}}\bigg\| \nabla_\Hdag \sum_{k=1}^{K-1}   T_{\bbWn}^{(k)}  \bigg(X^g_{ l-1}- X^g_{n l-1}\bigg)[\bbH_{lk}]_{gf})\bigg\|\nonumber .
	\eal
	Now note that as we consider the case in which $l_\dagger < l-1$, using the Cauchy-Schwartz inequality we can use the same bound for the first and second terms of the right hand side of \eqref{eqn:norm_bound_dif_grad_WNN_long_triangle_inequality}. Also note that, by Assumption \ref{as3}, the filters are non-expansive, which allows us to write
	\bal
	\|&\nabla_{\Hdag}X_{L}^f-\nabla_{\Hdag}X_{nL}^f\|\nonumber\\
	&\quad\quad\leq 2\bigg\|\sum_{g=1}^{F_{l-1}} \sum_{k=1}^{K-1} (T_{\bbW}^{(k)} X^g_{ l-1})[\bbH_{lk}]_{gf}\nonumber\\	
	&\quad\quad-\sum_{g=1}^{F_{l-1}} \sum_{k=1}^{K-1} (T_{\bbW_n}^{(k)} X^g_{n l-1})[\bbH_{lk}]_{gf}\bigg\| F^{L-1}\|X_0\|\nonumber\\	
	&\quad\quad+\sum_{g=1}^{F_{l-1}}\bigg\| \nabla_\Hdag   \bigg(X^g_{ l-1}- X^g_{n l-1}\bigg)\bigg\|\nonumber .
	\eal
	The only term that remains to bound has the exact same bound derived in \eqref{eq:norm_bound_dif__grad_WNN_first_equality}, but on the previous layer $L-2$. Hence, by applying the same steps $L-2$ times, 
	we can obtain a bound for any element $\Hdag$ of tensor $\ccalH$. 
    \bal\label{eqn:solved_recursion_in_lemma_2}
	\|&\nabla_{\Hdag}X_{L}^f-\nabla_{\Hdag}X_{nL}^f\|\\
	&\quad\quad\leq 2 L F^{L-2} \bigg\|\sum_{g=1}^{F_{l-1}} \sum_{k=1}^{K-1} (T_{\bbW}^{(k)} X^g_{ l-1})[\bbH_{lk}]_{gf}\nonumber\\	
	&\quad\quad-\sum_{g=1}^{F_{l-1}} \sum_{k=1}^{K-1} (T_{\bbW_n}^{(k)} X^g_{n l-1})[\bbH_{lk}]_{gf}\bigg\| F^{L-1}\|X_0\|\nonumber\\
	&\quad\quad+\sum_{g=1}^{F_{l-1}}\bigg\| \nabla_\Hdag   \bigg(X^g_{ 1}- X^g_{1}\bigg)\bigg\|\nonumber .
	\eal
	Note that the derivative of a convolutional filter $T_\bbH$ at coefficient $k^{\dagger}=i$ is itself a convolutional filter with coefficients $\bbh_i$ [cf. Definition \ref{eq:graphon_convolution}]. The values of $\bbh_i$ are $[\bbh_i]_j=1$ if $j=i$ and $0$ otherwise. Additionally, this $\bbh_i$ is itself a filter that verifies Assumption \ref{as2}, as graphons are normalized. Thus, considering that $l^\dagger =0$, and using Propositions \ref{prop:x-x'}, \ref{prop:w-w'}, \cite[Theorem 1]{chung2011spectra} together with the triangle inequality, we obtain
	\bal\label{eqn:filter_at_1_in_lemma_2}
	\bigg\|&{\bbh_i}_{*\bbW_n}X_{n0}- {\bbh_i}_{*\bbW}X_0 \bigg\|\\	
	&\leq\bigg(\| \bbW-\mbW_n\|+\| \mbW_n-\bbW_n\|\bigg)\|X_0 \|+\|X_{n0}-X_0 \|\nonumber\\
	&\leq\bigg(1+\frac{\pi B_{\bbW_n}^c}{\delta_{\bbW\bbW_n}^c} \bigg)\frac{2\big(1+\sqrt{n \log({2n}/{\xi})} \big)}{n}+\frac{1}{n}\nonumber
	\eal
	with probability $1-\xi$, where $\mbW_n$ is the template graphon. Now, substituting \eqref{eqn:solved_recursion_in_lemma_2} into \eqref{eqn:filter_at_1_in_lemma_2}, and using Lemma \ref{lemma1:WNN-WNNn}, with probability $1-\xi$, it holds that
	\bal\label{lemma2:last}
	\|&\nabla_{\Hdag}X_{L}^f-\nabla_{\Hdag}X_{nL}^f\|\nonumber\\	
	&\quad\leq 2 L^2F^{2L-2}\bigg(1+\frac{\pi B_{\bbW_n}^c}{\delta_{\bbW\bbW_n}^c} \bigg)\frac{2\big(1+\sqrt{n \log({2n}/{\xi})} \big)}{n}\nonumber\\
    &\quad+\frac{2 F^{L-1}L}{n}+8L^2F^{2L-2}c .\nonumber
	\eal
	To achieve the final result, note that tensor $\ccalH$ has $KF^{L-1}$ elements, and each element is upper bounded by \eqref{lemma2:last}.
	\end{proof}
	
	\begin{lemma}\label{lemma:norm_bound_dif_loss_WNN}
	Let $\bbPhi(X;\ccalH,\bbW)$ be a WNN with $F_0=F_L=1$, and $F_l=F$ for $1\leq l \leq L-1$. Let $c \in (0,1]$, and assume that the graphon convolutions in all layers of this WNN have $K$ filter taps [cf. \eqref{eq:graphon_convolution}]. Let $\bbPhi(\bbx_n;\ccalH,\bbS_n)$ be a GNN sampled from $\bbPhi(X;\ccalH,\bbW)$ as in \eqref{eqn:gcn_obtained}. Under Assumptions \ref{as1}--\ref{as5}, with probability $1-\xi$ it holds that
	\bal
	\|&\nabla_{\ccalH}\ell(Y,\bbPhi(X;\ccalH,\bbW))-\nabla_{\ccalH}\ell(Y_n,\bbPhi(X_n;\ccalH,\bbW_n))\|\\
	&\quad\leq \sqrt{KF^{L-1}}\bigg(\frac{4 F^{L-1}L}{n}+12L^2F^{2L-2}c\nonumber\\
    &\quad+3 L^2F^{2L-2}\bigg(1+\frac{\pi B_{\bbW_n}^c}{\delta_{\bbW\bbW_n}^c} \bigg)\frac{2\big(1+\sqrt{n \log({2n}/{\xi})} \big)}{n}\bigg) .\nonumber 
    \eal
\end{lemma}
\begin{proof}
	In order to analyze the norm of the gradient with respect to the tensor $\ccalH$, we start by taking the derivative with respect to a single element of the tensor, $\Hdag$. Using the chain rule to compute the gradient of the loss function $\ell$, we get
	\bal
	\|&\nabla_{\Hdag}(\ell(Y,\bbPhi(X;\ccalH,\bbW))-\ell(Y_n,\bbPhi(X_n;\ccalH,\bbW_n)))\|\nonumber\\
	&\quad=\|\nabla \ell(Y,\bbPhi(X;\ccalH,\bbW))\nabla_{\Hdag}\bbPhi(X;\ccalH,\bbW)\nonumber\\
	&\quad-\nabla\ell(Y_n,\bbPhi(X_n;\ccalH,\bbW_n))\nabla_{\Hdag}\bbPhi(X_n;\ccalH,\bbW_n)\|\nonumber
	\eal
	and by Cauchy-Schwartz and the triangle inequality, it holds
	\bal
	\|&\nabla_{\Hdag}(\ell(Y,\bbPhi(X;\ccalH,\bbW))-\ell(Y_n,\bbPhi(X_n;\ccalH,\bbW_n)))\|\nonumber\\
	&\quad\leq\|\nabla \ell(Y,\bbPhi(X;\ccalH,\bbW))-\nabla\ell(Y_n,\bbPhi(X_n;\ccalH,\bbW_n))\|\nonumber\\
	&\quad\|\nabla_{\Hdag}\bbPhi(X;\ccalH,\bbW)\|+\|\nabla\ell(Y_n,\bbPhi(X_n;\ccalH,\bbW_n))\|\nonumber\\ &\quad\|\nabla_{\Hdag}\bbPhi(X;\ccalH,\bbW)-\nabla_{\Hdag}\bbPhi(X_n;\ccalH,\bbW_n)\|\nonumber.
	\eal
	By the triangle inequality and Assumption \ref{as4}, it follows
	\bal
	\|&\nabla_{\Hdag}(\ell(Y,\bbPhi(X;\ccalH,\bbW))-\ell(Y_n,\bbPhi(X_n;\ccalH,\bbW_n)))\|\nonumber\\
	&\quad\quad\leq\|\nabla \ell(Y,\bbPhi(X;\ccalH,\bbW))-\nabla\ell(Y,\bbPhi(X_n;\ccalH,\bbW_n))\|\nonumber\\ &\quad\quad\times\|\nabla_{\Hdag}\bbPhi(X;\ccalH,\bbW)\|\nonumber\\	
	&\quad\quad\times\|\nabla \ell(Y_n,\bbPhi(X_n;\ccalH,\bbW_n))-\nabla\ell(Y,\bbPhi(X_n;\ccalH,\bbW_n))\|\nonumber\\ 
	&\quad\quad\times\|\nabla_{\Hdag}\bbPhi(X;\ccalH,\bbW)\|\nonumber\\	
	&\quad\quad+\|\nabla_{\Hdag}(\bbPhi(X;\ccalH,\bbW)-\bbPhi(X_n;\ccalH,\bbW_n))\|\nonumber\\
	&\quad\quad\leq(\|Y_n-Y \|+\|\bbPhi(X_n;\ccalH,\bbW_n))-\bbPhi(X;\ccalH,\bbW)) \|)\nonumber\\ &\quad\quad\times\|\nabla_{\Hdag}\bbPhi(X;\ccalH,\bbW)\|\nonumber\\	
	&\quad\quad+ \|\nabla_{\Hdag}(\bbPhi(X;\ccalH,\bbW)-\bbPhi(X_n;\ccalH,\bbW_n))\|\nonumber.
	\eal
	
	Next, we can use Lemmas \ref{lemma:norm_grad_on_wnn_bound}, \ref{lemma1:WNN-WNNn}, and \ref{lemma:norm_bound_dif_grad_WNN}, Proposition \ref{prop:x-x'}, and Assumption \ref{as1} to obtain 	
	\bal\label{eq:last_lemma3}
	\|&\nabla_{\Hdag}(\ell(Y,\bbPhi(X;\ccalH,\bbW))-\ell(Y_n,\bbPhi(X_n;\ccalH,\bbW_n)))\|\nonumber\\
	&\quad\quad\quad\leq \bigg(3 L^2F^{2L-2}\bigg(1+\frac{\pi B_{\bbW_n}^c}{\delta_{\bbW\bbW_n}^c} \bigg)\frac{2\bigg(1+\sqrt{n \log(\frac{2n}{\xi})} \bigg)}{n}\nonumber\\
    &\quad\quad\quad+\frac{4 F^{L-1}L}{n}+12L^2F^{2L-2}c\bigg).
	\eal
	
	Noting that tensor $\ccalH$ has $KF^{L-1}$ elements, and that each individual term can be bounded by \eqref{eq:last_lemma3}, we arrive at the desired result.
\end{proof}	

\begin{definition}\label{def:Constant_Of_GNN}
We define the constant $\gamma$ as
\begin{align}
\gamma =12  \sqrt{K F^{L-1}}  L^2 F^{2L-2}
\end{align}
where $F$ is the number of features, $L$ is the number of layers, and $K$ is the number of filter taps of the GNN.
\end{definition}
\begin{proof}[Proof of Theorem \ref{thm:grad_Approximation}]
	To start, consider the event $A_n$ such that, 
	\bal
	A_n=&\bigg(\|\nabla_{\ccalH}(\ell(Y,\bbPhi(X;\ccalH,\bbW))-\ell(Y_n,\bbPhi(X_n;\ccalH,\bbW_n)))\|\nonumber\\
	&\leq\sqrt{KF^{L-1}}\bigg(\frac{4 F^{L-1}L}{n}+12L^2F^{2L-2}c\nonumber\\
	&+3 L^2F^{2L-2}\bigg(1+\frac{\pi B_{\bbW_n}^c}{\delta_{\bbW\bbW_n}^c} \bigg)\frac{2\big(1+\sqrt{n \log({2n}/{\xi})} \big)}{n}\bigg)\bigg).\nonumber
	\eal
	Taking the disjoint events $A_n$ and $A_n^c$, and denoting the indicator function $\bbone(\cdot)$, we split the expectation as
	\bal
	&\mbE[\|\nabla_{\ccalH}\ell(Y,\bbPhi(X;\ccalH,\bbW))-\nabla_{\ccalH}\ell(Y_n,\bbPhi(X_n;\ccalH,\bbW_n))\|]\nonumber\\
	&\quad\quad\quad\quad\quad\quad\quad=\mbE[\|\nabla_{\ccalH}(\ell(Y,\bbPhi(X;\ccalH,\bbW))\nonumber\\
	&\quad\quad\quad\quad\quad\quad\quad-\ell(Y_n,\bbPhi(X_n;\ccalH,\bbW_n)))\|\bbone(A_n)]\nonumber\\
	&\quad\quad\quad\quad\quad\quad\quad+\mbE[\|\nabla_{\ccalH}\ell(Y,\bbPhi(X;\ccalH,\bbW))\nonumber\\
	&\quad\quad\quad\quad\quad\quad\quad-\nabla_{\ccalH}\ell(Y_n,\bbPhi(X_n;\ccalH,\bbW_n))\|\bbone(A_n^c)]\label{eqn:Theo1_expectation_split}
	\eal
    We can then bound the term corresponding to $A_n^c$ using the chain rule, the Cauchy-Schwartz inequality, Assumption \ref{as4}, and Lemma \ref{lemma:norm_grad_on_wnn_bound} as follows 
	\begin{align}
	&\|\nabla_{\ccalH}\ell(Y,\bbPhi(X;\ccalH,\bbW))-\nabla_{\ccalH}\ell(Y_n,\bbPhi(X_n;\ccalH,\bbW_n))\|\nonumber\\
	&\quad\leq \|\nabla_{\ccalH}\ell(Y,\bbPhi(X;\ccalH,\bbW))\|+\|\nabla_{\ccalH}\ell(Y_n,\bbPhi(X_n;\ccalH,\bbW_n))\|\nonumber\\
	&\quad\leq \|\nabla\ell(Y,\bbPhi(X;\ccalH,\bbW))\| \|\nabla_{\ccalH}\bbPhi(X;\ccalH,\bbW)\|\nonumber\\
	&\quad+\|\nabla\ell(Y_n,\bbPhi(X_n;\ccalH,\bbW_n))\|\|\nabla_{\ccalH}\bbPhi(X_n;\ccalH,\bbW_n)\|\nonumber\\
	&\quad\leq \|\nabla_{\ccalH}\bbPhi(X;\ccalH,\bbW)\|+\|\nabla_{\ccalH}\bbPhi(X_n;\ccalH,\bbW_n)\|\nonumber\\
	&\quad\leq 2 F^{2L} \sqrt{K} . \label{eqn:Theo1_bound_A_c}
	\end{align}
	Going back to \eqref{eqn:Theo1_expectation_split}, we can substitute the bound obtained in \eqref{eqn:Theo1_bound_A_c}, take $P(A_n)=1-\xi$, and use Lemma \ref{lemma:norm_bound_dif_loss_WNN} to get
	\bal
	&\mbE[\|\nabla_{\ccalH}\ell(Y,\bbPhi(X;\ccalH,\bbW))-\nabla_{\ccalH}\ell(Y_n,\bbPhi(X_n;\ccalH,\bbW_n))\|]\nonumber\\
	&\leq \xi 2 F^{2L} \sqrt{K} +(1-\xi)\sqrt{KF^{L-1}}\bigg(\frac{4 F^{L-1}L}{n}+12L^2F^{2L-2}c\nonumber\\
	&+3 L^2F^{2L-2}\bigg(1+\frac{\pi B_{\bbW_n}^c}{\delta_{\bbW\bbW_n}^c} \bigg)\frac{2\big(1+\sqrt{n \log({2n}/{\xi})} \big)}{n}\bigg). \nonumber
	\eal
	Setting $\xi=\frac{1}{\sqrt{n}}$ completes the proof. 
\end{proof}
\section{Proof of Lemma \ref{lemma:martingale}}
\label{sec:AppendixB}

\begin{lemma}\label{lemma:loss_grad_lipschitz}
	Under Assumptions $\ref{as4},\ref{as5}$, and $\ref{as7}$, the gradient of the loss function $\ell$ with respect to the GNN parameters $\ccalH$ is $\Anl$-Lipschitz, with $\Anl=(\AnWNN+\AWNN F^{2L} \sqrt{K})$. Explicitly,
	\bal
	&\|\nabla_{\ccalH}\ell(Y,\bbPhi(X;\ccalA,\bbW)) -\nabla_{\ccalH}\ell(Y,\bbPhi(X;\ccalB,\bbW)) \|\nonumber\\
	&\quad\quad\quad\quad\quad\quad\quad\quad\quad\quad\quad\quad\quad\quad\leq \Anl \|\ccalA-\ccalB\|
	\eal
\end{lemma}
\begin{proof}
	We start by applying the chain rule to obtain
	\bal
	\|&\nabla_{\ccalH}\ell(Y,\bbPhi(X;\ccalA,\bbW)) -\nabla_{\ccalH}\ell(Y,\bbPhi(X;\ccalB,\bbW)) \|\nonumber\\
	&\quad\quad\quad=\|\nabla \ell(Y,\bbPhi(X;\ccalA,\bbW))\nabla_{\ccalH}\bbPhi(X;\ccalA,\bbW)\nonumber\\ &\quad\quad\quad-\nabla\ell(Y,\bbPhi(X;\ccalB,\bbW))\nabla_{\ccalH}\bbPhi(X;\ccalB,\bbW)\| \nonumber
	\eal
    followed by the triangle and Cauchy-Schwartz inequalities to get
	\bal
	\|&\nabla_{\ccalH}\ell(Y,\bbPhi(X;\ccalA,\bbW)) -\nabla_{\ccalH}\ell(Y,\bbPhi(X;\ccalB,\bbW)) \|\nonumber\\
	&\quad\quad\quad\leq\|\nabla_{\ccalH}\bbPhi(X;\ccalA,\bbW)-\nabla_{\ccalH}\bbPhi(X;\ccalB,\bbW)\|\nonumber\\
	&\quad\quad\quad\times\|\nabla\ell(Y,\bbPhi(X;\ccalA,\bbW))\|+\|\nabla_{\ccalH}\bbPhi(X;\ccalB,\bbW)\| .\nonumber\\
	&\quad\quad\quad\times\|\nabla_\ccalH\ell(Y,\bbPhi(X;\ccalA,\bbW))-\ell(Y,\bbPhi(X;\ccalB,\bbW))\|\nonumber
	\eal
	Finally, we use Assumptions $\ref{as1},\ref{as4},\ref{as5}$, and $\ref{as7}$, and Lemma \ref{lemma:norm_grad_on_wnn_bound} to write
	\bal
	\|&\nabla_{\ccalH}\ell(Y,\bbPhi(X;\ccalA,\bbW)) -\nabla_{\ccalH}\ell(Y,\bbPhi(X;\ccalB,\bbW)) \|\nonumber\\
	&\quad\quad\quad\quad\quad\quad\leq(\AnWNN+\AWNN F^{2L} \sqrt{K})\|\ccalA-\ccalB\|\nonumber
	\eal
	completing the proof. 		
\end{proof}
\begin{proof}[Proof of Lemma \ref{lemma:martingale}]
	To start, we do as in \cite{bertsekas2000gradient}, i.e., we define a continuous function $g(\epsilon)$ that at $\epsilon=1$ takes the value of the loss function on the graphon data at iteration $k+1$, and at $\epsilon=0$, the value at iteration $k$. Explicitly,
	\bal
	g(\epsilon)=\ell(Y,\bbPhi(X; \ccalH_{k}-\epsilon\eta_k \nabla_{\ccalH} \ell(Y_n,\bbPhi(X_n;\ccalH_k,\bbWn)),\bbW)).\nonumber
	\eal
	Function $g(\epsilon)$ is evaluated on the graphon data $Y,X,\bbW$, but the steps are controlled by the induced graphon data $Y_n,X_n,\bbW_n$. Applying the chain rule, the derivative of $g(\epsilon)$ with respect to $\epsilon$ can be written as
	\bal
	&\frac{\partial g(\epsilon)}{\partial \epsilon} =-\eta_k \nabla_{\ccalH} \ell(Y_n,\bbPhi(X_n;\ccalH_k;\bbWn))\nonumber\\
	&\quad\times\nabla_{\ccalH}\ell(Y,\bbPhi(X;\ccalH_{k}-\epsilon\eta_k \nabla_{\ccalH} \ell(Y_n,\bbPhi(\ccalH_k;\bbWn;X_n)),\bbW)).\nonumber
	\eal
    Between iterations $k+1$ and $k$, the difference in the loss function $\ell$ can be written as the difference between $g(1)$ and $g(0)$,
	\bal
	g(1)-g(0)=\ell(Y,\bbPhi(X;\ccalH_{k+1},\bbW))-\ell(Y,\bbPhi(X;\ccalH_{k},\bbW)).\nonumber
	\eal
	Integrating the derivative of $g(\epsilon)$ in $[0,1]$, we get
	\bal
	&\ell(Y,\bbPhi(X;\ccalH_{k+1},\bbW))-\ell(Y,\bbPhi(X;\ccalH_{k},\bbW))\nonumber\\
	&\quad\quad\quad\quad\quad\quad\quad\quad\quad=g(1)-g(0)=\int_0^1 \frac{\partial g(\epsilon)}{\partial \epsilon} d\epsilon\nonumber\\
	&\quad\quad\quad\quad\quad\quad\quad\quad\quad=-\int_0^1\eta_k \nabla_{\ccalH} \ell(Y_n,\bbPhi(X_n;\ccalH_k,\bbWn))\nonumber\\
	&\times\nabla_{\ccalH}\ell(Y,\bbPhi(X;\ccalH_{k}-\epsilon\eta_k \nabla_{\ccalH} \ell(Y_n,\bbPhi(X_n;\ccalH_k,\bbWn)),\bbW))d\epsilon.\nonumber
	\eal
	Now note that the last term of the integral does not depend on $\epsilon$. Thus, we can proceed by adding and subtracting $\nabla_{\ccalH}\ell(Y,\bbPhi(\ccalH_{k},\bbW,X))$ inside the integral to get
	\bal
	\ell&(Y,\bbPhi(X;\ccalH_{k+1},\bbW))-\ell(Y,\bbPhi(X;\ccalH_{k},\bbW))\nonumber\\
	&=-\eta_k \nabla_{\ccalH} \ell(Y_n,\bbPhi(X_n;\ccalH_k,\bbWn))\nonumber\\
	&\times\int_0^1 \nabla\ell(Y,\bbPhi(X;\ccalH_{k}-\epsilon\eta_k \nabla \ell(Y_n,\bbPhi(X_n;\ccalH_k,\bbWn)),\bbW))\nonumber\\
	&+\nabla_{\ccalH}\ell(Y,\bbPhi(X;\ccalH_{k},\bbW))-\nabla_{\ccalH}\ell(Y,\bbPhi(X;\ccalH_{k},\bbW))d\epsilon\nonumber\\
	&=-\eta_k \nabla_{\ccalH} \ell(Y_n,\bbPhi(X_n;\ccalH_k,\bbWn))\nabla_{\ccalH}\ell(Y,\bbPhi(X;\ccalH_{k},\bbW))\nonumber\\
	&-\eta_k \nabla_{\ccalH} \ell(Y_n,\bbPhi(X_n;\ccalH_k,\bbWn))\nonumber\\
	&\times\int_0^1 \nabla_{\ccalH}\ell(Y,\bbPhi(\ccalH_{k}-\epsilon\eta_k \nabla \ell(Y_n,\bbPhi(X_n;\ccalH_k,\bbWn)),\bbW,X))\nonumber\\
	&-\nabla\ell(Y,\bbPhi(X;\ccalH_{k},\bbW))d\epsilon. \label{eqn:Lemma1_before_bouding_integral}
	\eal
	
	Next, we can apply the Cauchy-Schwartz inequality to the last term of \eqref{eqn:Lemma1_before_bouding_integral} and take the norm of the integral (which is smaller that the integral of the norm), to obtain
	\bal
	\ell&(Y,\bbPhi(X;\ccalH_{k+1},\bbW))-\ell(Y,\bbPhi(X;\ccalH_{k},\bbW))\nonumber\\
	&\leq-\eta_k \nabla_{\ccalH} \ell(Y_n,\bbPhi(\ccalH_k;\bbWn;X_n))\nabla_{\ccalH}\ell(Y,\bbPhi(\ccalH_{k},\bbW,X))\nonumber\\
	&+\eta_k\|\nabla_{\ccalH} \ell(Y_n,\bbPhi(\ccalH_k;\bbWn;X_n))\|\nonumber\\
	&\times \int_0^1\|\nabla_{\ccalH}\ell(Y,\bbPhi(\ccalH_{k},\bbW,X))\nonumber\\
	&-\nabla\ell(Y,\bbPhi(\ccalH_{k}-\epsilon\eta_k \nabla \ell(Y_n,\bbPhi(\ccalH_k;\bbWn;X_n)),\bbW,X))\|d\epsilon.\nonumber\eal
	
	By Lemma \ref{lemma:loss_grad_lipschitz}, we use $\Anl$ to write
	\bal
	\ell(Y&,\bbPhi(X;\ccalH_{k+1},\bbW))-\ell(Y,\bbPhi(X;\ccalH_{k},\bbW))\nonumber\\
	&\leq-\eta_k \nabla_{\ccalH} \ell(Y_n,\bbPhi(X_n;\ccalH_k,\bbWn))\nabla_{\ccalH}\ell(Y,\bbPhi(X;\ccalH_{k},\bbW))\nonumber\\
	&+\Anl\eta_k \|\nabla_{\ccalH} \ell(Y_n,\bbPhi(X_n;\ccalH_k,\bbWn))|\nonumber \\
	&\times \int_0^1 \bigg\| \eta_k \nabla_{\ccalH} \ell(Y_n,\bbPhi(X_n;\ccalH_k,\bbWn))\bigg\|\epsilon d\epsilon\nonumber\\
	&\leq-\eta_k \nabla_{\ccalH} \ell(Y_n,\bbPhi(X_n;\ccalH_k,\bbWn))\nabla_{\ccalH}\ell(Y,\bbPhi(X;\ccalH_{k},\bbW))\nonumber\\
	&+\frac{\eta_k^2 \Anl}{2} \|\nabla_{\ccalH} \ell(Y_n,\bbPhi(X_n;\ccalH_k,\bbWn))\|^2. \nonumber
	\eal
	
	Factoring out $\eta_k$, we get
	\bal
	\ell(Y&,\bbPhi(X;\ccalH_{k+1},\bbW))-\ell(Y,\bbPhi(X;\ccalH_{k},\bbW))\nonumber\\
	&\leq -\frac{\eta_k}{2}  \bigg(-\|\nabla_{\ccalH} \ell(Y_n,\bbPhi(X_n;\ccalH_k,\bbWn))\|^2  \nonumber\\
	&+2\nabla_{\ccalH}\ell(Y_n,\bbPhi(X_n;\ccalH_k,\bbWn))^T\nabla_{\ccalH}\ell(Y,\bbPhi(X;\ccalH_{k},\bbW))\bigg) \nonumber\\
	&+\frac{\eta_k^2 \Anl-\eta_k}{2} \|\nabla_{\ccalH} \ell(Y_n,\bbPhi(X_n;\ccalH_k,\bbWn))\|^2. \nonumber
	\eal
	
	Given that the norm is induced by the vector inner product in $\reals^{LF}$, for any two vectors $A,B$, $||A-B||^2-||B||^2=||A||^2-2\langle A,B\rangle$. Hence,
	\bal
	\ell&(Y,\bbPhi(X;\ccalH_{k+1},\bbW))-\ell(Y,\bbPhi(X;\ccalH_{k},\bbW))\nonumber\\
	&\leq \frac{-\eta_k}{2}  \bigg(\|\nabla_{\ccalH} \ell(Y,\bbPhi(X;\ccalH_k,\bbW))\|^2  \nonumber\\
	&-||\nabla_{\ccalH}\ell(Y_n,\bbPhi(X_n;\ccalH_k,\bbWn))-\nabla_{\ccalH}\ell(Y,\bbPhi(X;\ccalH_{k},\bbW))||^2\bigg) \nonumber\\
	&+\frac{\eta_k^2 \Anl-\eta_k}{2} \|\nabla_{\ccalH} \ell(Y_n,\bbPhi(X_n;\ccalH_k,\bbWn))\|^2.\nonumber
	\eal
	Considering the first term on the right hand side, we know that the norm of the expected difference between the gradients is bounded by Theorem \ref{thm:grad_Approximation}. Given that norms are positive, the inequality still holds when the elements are squared (if $a>b,a\in \reals_+,b\in\reals_+$, then $a^2>b^2$). Considering the second term on the right hand side, we impose the condition that $\eta_k<\frac{1}{\Anl}$, which makes this term negative. Taking the expected value of all the nodes completes the proof.
\end{proof}

\bibliographystyle{IEEEtran}
\bibliography{bib}

\end{document}